\documentclass[12pt]{article} 

\usepackage[round]{natbib}
\usepackage{times}
\usepackage{authblk}

\usepackage{microtype}
\usepackage{booktabs} 
\usepackage{subcaption}
\usepackage[toc,page]{appendix}

\usepackage[utf8]{inputenc} 
\usepackage[T1]{fontenc}    
\usepackage{amsfonts}       
\usepackage{nicefrac}       

\usepackage{xcolor}
\usepackage{url}
\usepackage{amsmath, amsthm}
\usepackage{amsfonts}
\usepackage{amssymb}
\usepackage{mathtools}
\usepackage{stmaryrd}
\usepackage{dsfont}
\usepackage{tikz}
\usetikzlibrary{decorations.pathreplacing}
\usepackage{algorithm}
\usepackage{algorithmic}
\usepackage{multirow}
\allowdisplaybreaks

\renewcommand{\epsilon}{\varepsilon}

\DeclareMathOperator{\Exp}{exp}
\DeclareMathOperator{\Expec}{\mathbb{E}}
\DeclareMathOperator{\Indic}{\mathds{1}}
\DeclareMathOperator{\Proba}{\mathbb{P}}
\DeclareMathOperator{\Var}{Var}

\newcommand{\Defeq}{\vcentcolon =}
\renewcommand{\exp}[1]{\Exp\left(#1\right)}
\newcommand{\expec}[1]{\Expec\left[#1\right]}
\newcommand{\indic}[1]{\Indic_{#1}}
\newcommand{\proba}[1]{\Proba\left (#1\right )}
\newcommand{\var}[1]{\Var\left(#1\right)}
\newtheorem{thm}{Theorem}
\newtheorem{lm}{Lemma}

\newtheorem{defn}{Definition}

\newtheorem{prop}{Proposition}

\newcommand{\hide}[1]{\textcolor{black}{}}

\title{Exponentially convergent stochastic $k$-PCA without variance reduction}

\author{Cheng Tang\thanks{tcheng@amazon.com}}
\affil{\small{Amazon AI}}
\date{}

\begin{document}

\maketitle

\begin{abstract}%
We present \textit{Matrix Krasulina}, an algorithm for online $k$-PCA, by generalizing the classic Krasulina's method \citep{krasulina1969} from vector to matrix case. 
We show, both theoretically and empirically, that the algorithm naturally adapts to data low-rankness and converges exponentially fast to the ground-truth principal subspace.
Notably, our result suggests that despite various recent efforts to accelerate the convergence of stochastic-gradient based methods by adding a $O(n)$-time variance reduction step, for the $k$-PCA problem, a truly online SGD variant suffices to achieve exponential convergence on intrinsically low-rank data.
\end{abstract}


\section{Introduction}
\label{sec:intro}
Principal Component Analysis (PCA) is ubiquitous in statistics, machine learning, and engineering alike:
For a centered $d$-dimensional random vector $X\in\mathbb{R}^d$, the $k$-PCA problem is defined as finding the ``optimal'' projection of the random vector into a subspace of dimension $k$ so as to capture as much of its variance as possible; 
formally, we want to find a rank $k$ matrix $W$ such that
\[
\max_{W\in \mathbb{R}^{k\times d}, WW^{\top}=I_{k}} \var{W^{\top}W X}
\]
In the objective above, $W^{\top}W=W^{\top}(WW^{\top})^{-1}W$ is an orthogonal projection matrix into the subspace spanned by the rows of $W$.
Thus, the $k$-PCA problem seeks matrix $W$ whose row-space captures as much variance of $X$ as possible.
This is equivalent to finding a projection into a subspace that minimizes variance of data outside of it: 
\begin{equation}
\label{eq:pca-population}
\min_{W\in \mathbb{R}^{k\times d}, WW^{\top}=I_k} \Expec \|X - W^{\top}WX\|^2
\end{equation}
Likewise, given a sample of $n$ centered data points $\{X_i\}_{i=1}^n$, the empirical version of problem \eqref{eq:pca-population} is
\begin{equation}
\label{eq:pca-empirical}
\min_{ W\in \mathbb{R}^{k\times d}, WW^{\top}=I_k} \frac{1}{n} \sum_{i=1}^n \|X_i - W^{\top}WX_i\|^2
\end{equation}
The optimal $k$-PCA solution, the row space of optimal $W$, can be used to represent high-dimensional data in a low-dimensional subspace ($k \ll d$), since it preserves most variation from the original data. As such, it usually serves as the first step in exploratory data analysis or as a way to compress data before further operation. 

The solutions to the nonconvex problems \eqref{eq:pca-population} and \eqref{eq:pca-empirical} are the subspaces spanned by the top $k$ eigenvectors (also known as the \textit{principal subspace}) of the population and empirical data covariance matrix, respectively. Although we do not have access to the population covariance matrix to directly solve \eqref{eq:pca-population}, given a batch of samples $\{x_i\}_{i=1}^n$ from the same distribution, we can find the solution to \eqref{eq:pca-empirical}, which asymptotically converges to the population $k$-PCA solution \citep{samplePCA:finite-sample}. Different approaches exist to solve \eqref{eq:pca-empirical} depending on the nature of the data and the computational resources available:

\paragraph{SVD-based solvers} When data size is manageable, one can find the exact solution to \eqref{eq:pca-empirical} via a singular value decomposition (SVD) of the empirical data matrix in $\min\{O(nd^2), O(n^2d)\}$-time and $O(nd)$-space, or in case of truncated SVD in $O(ndk)$-time ($O(nd\log k)$ for randomized solver \citep{randomized-svd:Halko:2011}).
%
\paragraph{Power method}
For large-scale datasets, that is, both $n$ and $d$ are large, the full data may not fit in memory.
Power method \citep[p.450]{Golub:1996} and its variants are popular alternatives in this scenario; they have less computational and memory burden than SVD-based solvers; power method approximates the principal subspace iteratively: 
At every iteration, power method computes the inner product between the algorithm's current solution and $n$ data vectors
$\{x_i\}_{i=1}^n$, an $O(nd_s)$-time operation, where $d_s$ is the average data sparsity. 
Power method converges exponentially fast \citep{Shamir:exp-pca:2015}: To achieve $\epsilon$ accuracy, it has a total runtime of $O(nd_s\log\frac{1}{\epsilon})$. That is, power method requires multiple passes over the full dataset.

\paragraph{Online (incremental) PCA} 
In real-world applications, datasets might become so large that even executing a full data pass is impossible. Online learning algorithms are developed under an abstraction of this setup: They assume that data come from an ``endless stream'' and only process one data point (or a constant sized batch) at a time.
Online PCA mostly fall under two frameworks: 
1. The online worst-case scenario, where the stream of data can have a non-stationary distribution \citep{Nie:online-pca:2016, Boutsidis:online-pca:2015, Warmuth:online-pca:2006}.
2. The stochastic scenario, where one has access to i.i.d. samples from an unknown but fixed distribution \citep{Shamir:exp-pca:2015, balsubramani13, Mitliagkas:streaming-pca:2013, Arora:stochastic-pca:2013}. 

In this paper, we focus on the stochastic setup: We show that a simple variant of stochastic gradient descent (SGD), which generalizes the classic Krasulina's algorithm from $k=1$ to general $k\ge 1$, can provably solve the $k$-PCA problem in Eq. \eqref{eq:pca-population} with an exponential convergence rate. 
It is worth noting that stochastic PCA algorithms, unlike batch-based solvers, can be used to optimize both the population PCA objective \eqref{eq:pca-population} and its empirical counterpart \eqref{eq:pca-empirical}.

\paragraph{Oja's method and VR-PCA}
While SGD-type algorithms have iteration-wise runtime independent of the data size,
their convergence rate, typically linear in the number of iterations, is significantly slower than that of batch gradient descent (GD). 
To speed up the convergence of SGD, the seminal work of \citet{svrg} initiated a line of effort in deriving Variance-Reduced (VR) SGD by cleverly mixing the stochastic gradient updates with occasional batch gradient updates. 
For convex problems, VR-SGD algorithms have provable exponential convergence rate.
Despite the non-convexity of $k$-PCA problem,
\citet{Shamir:exp-pca:2015, Shamir:general-oja:2016} augmented Oja's method \citep{oja1982}, a popular stochastic version of power method, with the VR step, and showed both theoretically and empirically that the resulting VR-PCA algorithm achieves exponential convergence. 
However, since a single VR iteration requires a full-pass over the dataset, VR-PCA is no longer an online algorithm.
\paragraph{Minimax lower bound}
In general, the tradeoff between convergence rate and iteration-wise computational cost is unavoidable in light of the minimax information lower bound \citep{Vu:minimax-subspace:2013, vu:pca-lowerbound:2012}: Let $\Delta^n$ (see Definition \ref{defn:loss_measure}) denote the distance between the ground-truth rank-$k$ principal subspace and the algorithm's estimated subspace after seeing $n$ samples. \citet[Theorem 3.1]{Vu:minimax-subspace:2013} established that \textit{there exists data distribution (with full-rank covariance matrices)} such that the following lower bound holds:
\begin{eqnarray}
\label{eqn:lower-bound}
\expec{\Delta^n}\ge\Omega(\frac{\sigma^2}{n}) \text{~~for~}\sigma^2 \ge \frac{\lambda_1 \lambda_{k+1}}{(\lambda_k - \lambda_{k+1})^2}\, ,
\end{eqnarray}
Here $\lambda_k$ denotes the $k$-th largest eigenvalue of the data covariance matrix. This immediately implies a $\Omega(\frac{\sigma^2}{t})$ lower bound on the convergence rate of online $k$-PCA algorithms, since for online algorithms the number of iterations $t$ equals the number of data samples $n$.
Thus, sub-linear convergence rate is impossible for online $k$-PCA algorithms on general data distributions.

%
%
%
%

%
\subsection{Our result: escaping minimax lower bound on intrinsically low rank data}
Despite the discouraging lower bound for online $k$-PCA, note that in Eq.~\eqref{eqn:lower-bound}, $\sigma$ equals zero when the data covariance has rank less than or equal to $k$, and consequently, the lower bound becomes un-informative. Does this imply that data low-rankness can be exploited to overcome the lower bound on the convergence rate of online $k$-PCA algorithms?

Our result answers the question affirmatively: 
Theorem \ref{thm:main} suggests that on low-rank data, an online $k$-PCA algorithm, namely, Matrix Krasulina (Algorithm \ref{algo:kpca}), produces estimates of the principal subspace that locally converges to the ground-truth in order $O(\exp{-Ct})$, where $t$ is the number of iterations (the number of samples seen) and $C$ is a constant.
Our key insight is that Krasulina's method \citep{krasulina1969}, in contrast to its better-studied cousin Oja's method \citep{oja1982}, 
is stochastic gradient descent with a self-regulated gradient for the PCA problem, and that when the data is of low-rank, the gradient variance vanishes as the algorithm's performance improves.

In a broader context, our result is an example of ``learning faster on easy data'', a phenomenon widely observed for online learning \citep{adaptive-online:Beygelzimer2015}, clustering \citep{kumar}, and active learning \citep{noise-adaptive:Wang2016}, to name a few.
While low-rankness assumption has been widely used to regularize solutions to matrix completion problems \citep{Jain:AM:2013, keshavan:2010, candes:2009} and to model the related robust PCA problem \citep{nonconvex-rpca14, rpca-Candes:2011}, we are unaware of previous such methods that exploit data low-rankness to significantly reduce computation. 

\section{Preliminaries}
\label{sec:bak}
We consider the following online stochastic learning setting: At time $t\in \mathbb{N}\setminus\{0\}$, we receive a random vector $X^t\in \mathbb{R}^d$ drawn i.i.d from an unknown centered probability distribution with a finite second moment. We denote by $X$ a generic random sample from this distribution. Our goal is to learn $W\in \mathbb{R}^{k^{\prime}\times d}$ so as to optimize the objective in Eq~\eqref{eq:pca-population}. 
%
\paragraph{Notations}
We let $\Sigma^*$ denote the covariance matrix of $X$,
$
\Sigma^* \Defeq \expec{XX^{\top}} \, .
$
%
%
%
%
We let $\{u_i\}_{i=1}^k$ denote the top $k$ eigenvectors of covariance matrix $\Sigma^*$, corresponding to its largest $k$ eigenvalues, $\lambda_1\ge, \dots, \ge \lambda_k$. 
Given that $\Sigma^*$ has rank $k$, we can represent it by its top $k$ eigenvectors:
$
\Sigma^*\Defeq  \sum_{i=1}^k\lambda_i u_i u_i^{\top} \, .
$
We let 
$U^* \Defeq \sum_{i=1}^k u_i u_i^{\top} \, .$ 
That is, $U^*$ is the orthogonal projection matrix into the subspace spanned by $\{u_i\}_{i=1}^k$.
For any integer $p>0$, we let $I_p$ denote the $p$-by-$p$ identity matrix.
We denote by $\|\cdot\|_F$ the Frobenius norm, by $tr(\cdot)$ the trace operator. 
For two square matrices $A$ and $B$ of the same dimension, we denote by $A\succeq B$ if $A-B$ is positive semidefinite.
We use curly capitalized letters such as $\mathcal{G}$ to denote events. For an event $\mathcal{G}$, we denote by $\indic{\mathcal{G}}$ its indicator random variable; that is, 
$\indic{\mathcal{G}}=1$ if event $\mathcal{G}$ occurs and $0$ otherwise.
%
\paragraph{Optimizing the empirical objective}
We remark that our setup and theoretical results apply not only to the optimization of population $k$-PCA problem~\eqref{eq:pca-population} in the infinite data stream scenario, but also to the empirical version~\eqref{eq:pca-empirical}: Given a finite dataset, we can simulate the stochastic optimization setup by sampling uniformly at random from it. This is, for example, the setup adopted by \citet{Shamir:general-oja:2016, Shamir:exp-pca:2015}.
\paragraph{Assumptions}
In our analysis, we assume that $\Sigma^*$ has low rank and that the data norm is bounded almost surely; that is, there exits $b$ and $k$ such that
\begin{eqnarray}
\label{eq:assumption}
\proba{\sup_X \|X\|^2 > b} = 0
\mbox{~~and~~}
rank(\Sigma^*) = k
\end{eqnarray}

\subsection{Oja and Krasulina}
%
In this section, we introduce two classic online algorithms for $1$-PCA, Oja's method and Krasulina's method.

%
\paragraph{Oja's method}
Let $w^t \in \mathbb{R}^d$ denote the algorithm's estimate of the top eigenvector of $\Sigma^*$ at time $t$.
Then letting $\eta^t$ denote learning rate, and $X$ be a random sample, Oja's algorithm has the following update rule:
\[
w^t \leftarrow w^{t-1} + \eta^t (XX^{\top} w^{t-1})
\text{~~and~~}
w^t \leftarrow \frac{w^t}{\|w^t\|}
\]
We see that Oja's method is a stochastic approximation algorithm to power method.
For $k>1$, Oja's method can be generalized straightforwardly, by replacing $w^t$ with matrix $W^t\in\mathbb{R}^{k\times d}$, and by replacing the normalization step with row orthonormalization, for example, by QR factorizaiton.
\paragraph{Krasulina's method}
Krasulina's update rule is similar to Oja's update but has an additional term:
\begin{eqnarray*}
w^t 
\leftarrow 
w^{t-1} + \eta^t (XX^{\top} w^{t-1} - w^{t-1} (X^{\top} \frac{w^{t-1}}{\|w^{t-1}\|})^2 ) 
\end{eqnarray*}

In fact, this is stochastic gradient descent on the objective function below, which is equivalent to Eq \eqref{eq:pca-population}:
\[
\Expec \|X - \frac{w^t(w^t)^{\top}}{\|w^t\|^2}X\|^2
\]

We are unaware of previous work that generalizes Krasulina's algorithm to $k>1$.
\subsection{Gradient variance in Krasulina's method}
Our key observation of Krasulina's method is as follows:
Let $\tilde w^t\Defeq \frac{w^t}{\|w^t\|}$; Krasulina's update can be re-written as
\[
w^t 
\leftarrow 
w^{t-1} + \|w^t\|\eta^t (XX^{\top} \tilde w^{t-1} - \tilde w^{t-1} (X^{\top} \tilde w^{t-1})^2 ) 
\]
Let 
\[
s^t\Defeq (\tilde w^t)^{\top} X \text{~~(projection coefficient)} \, 
\] 
and
\[
r^t \Defeq X^{\top} - s^t (\tilde w^t)^{\top} = X^{\top} - (\tilde w^t)^{\top} X (\tilde w^t)^{\top} \text{~~(projection residual)}\, 
\] 
Krasulina's algorithm can be further written as:
\begin{eqnarray*}
w^t \leftarrow w^{t-1} + \|w^t\|\eta^t s^{t-1} (r^{t-1})^{\top}
\end{eqnarray*}
%
The variance of the stochastic gradient term can be upper bounded as:
\[
\|w^t\|^2 \var{s^{t-1} (r^{t-1})^{\top}}
\le
\|w^t\|^2
\sup_X \|X\|^2
\Expec \|r^t\|^2
\]
Note that
\[
\Expec \|r^t\|^2 = \Expec \|X - \frac{w^t(w^t)^{\top}}{\|w^t\|^2}X\|^2
\]
This reveals that the variance of the gradient naturally decays as Krasulina's method decreases the $k$-PCA optimization objective.
Intuitively, as the algorithm's estimated (one-dimensional) subspace $w^t$ gets closer to the ground-truth subspace $u_1$, $(w^t)^{\top}X$ will capture more and more of $X$'s variance, and $\Expec \|r^t\|^2$ eventually vanishes.

In our analysis, we take advantage of this observation to prove the exponential convergence rate of Krasulina's method on low rank data.

\begin{algorithm}[t]
   \caption{Matrix Krasulina's method}
   \label{algo:kpca}
\begin{algorithmic}
   \STATE {\bfseries Input:} 
   Initial matrix $W^o\in\mathbb{R}^{k^{\prime}\times d}$; learning rate schedule $(\eta^t)$; number of iterations, $T$;
   \WHILE{$t\leq T$}
   \STATE{1.}  Sample $X^t$ i.i.d. from the data distribution
   \STATE{}
    \STATE{2.}   Orthonormalize the rows of $W^{t-1}$ (e.g., via QR factorization)
    \STATE{} 
    \STATE{3.} 
   $
   W^{t} \leftarrow W^{t-1} + \eta^{t} W^{t-1} X^t (X^t - (W^{t-1})^{\top}W^{t-1} X^t)^{\top}
   $
   \ENDWHILE
    \STATE {\bfseries Output:} $W^{\top}$
\end{algorithmic}
\end{algorithm}

\section{Main results}
%
%
%
%
%
%
Generalizing vector $w^t \in\mathbb{R}^d$ to matrix $W^t \in \mathbb{R}^{k^{\prime} \times d}$ as the algorithm's estimate at time $t$, we derive \textit{Matrix Krasulina's method} (Algorithm \ref{algo:kpca}), so that the row space of $W^t$ converges to the $k$-dimensional subspace spanned by $\{u_1, \dots, u_k\}$.
%
%
\paragraph{Matrix Krasulina's method}
Inspired by the original Krasulina's method, we design the following update rule for the Matrix Krasulina's method (Algorithm \ref{algo:kpca}): Let 
\[
s^t \Defeq W^{t-1} X^t \, \text{~~and~~} r^t \Defeq X^t- (W^{t-1})^{\top}(W^{t-1}(W^{t-1})^{\top})^{-1} W^{t-1} X^t \, ,
\]
Since we impose an orthonormalization step in Algorithm \ref{algo:kpca}, $r^t$ is simplified to
\vspace{-0.2cm}
\[
r^t \Defeq X^t- (W^{t-1})^{\top} W^{t-1} X^t
\]
Then the update rule of Matrix Krasulina's method can be re-written as
\vspace{-0.2cm}
\[
W^{t} \leftarrow W^{t-1} + \eta^{t} s^t (r^t)^{\top}
\]
For $k^{\prime}=1$, this reduces to Krasulina's update with $\|w^t\|=1$.
The self-regulating variance argument for the original Krasulina's method still holds, that is, we have
\begin{eqnarray*}
\Expec \|s^t (r^t)^{\top}\|^2
\le
b \Expec \|r^t\|^2 
=
b \Expec \|X - (W^t)^{\top}W^t X\|^2
\end{eqnarray*}
where $b$ is as defined in Eq~\eqref{eq:assumption}.
We see that
the last term coincides with the objective function in Eq.~\eqref{eq:pca-population}.

\paragraph{Loss measure}
Given the algorithm's estimate $W^t$ at time $t$, 
we let $P^t$ denote the orthogonal projection matrix into the subspace spanned by its rows, $\{W^t_{i,\star}\}_{i=1}^{k^{\prime}}$, that is,
\[
P^t \Defeq (W^t)^{\top} (W^t (W^t)^{\top})^{-1} W^t = (W^t)^{\top} W^t
\, ,
\]
In our analysis, we use the following loss measure to track the evolvement of $W^t$:
\begin{defn}[Subspace distance]
\label{defn:loss_measure}
Let $\mathcal{S}$ and $\mathcal{\hat S}^t$ be the ground-truth principal subspace and its estimate of Algorithm \ref{algo:kpca} at time $t$ with orthogonal projectors $U^*$ and $P^t$, respectively. We define the subspace distance between $S$ and $\hat S^t$ as
$
\Delta^t \Defeq tr(U^* (I- P^t)) 
=
k - tr(U^* P^t)
$.
\end{defn}
Note that $\Delta^t$ in fact equals the sum of squared canonical angles between $\mathcal{S}$ and $\mathcal{\hat S}^t$, and coincides with the subspace distance measure used in related theoretical analyses of $k$-PCA algorithms \citep{allen-zhu:stochastic-pca:2017, Shamir:general-oja:2016, Vu:minimax-subspace:2013}.
%
In addition, $\Delta^t$ is related to the $k$-PCA objective function defined in Eq.~$\eqref{eq:pca-population}$ as follows (proved in Appendix Eq~\eqref{eqn:loss_measure}):
\[
\lambda_k \Delta^t \le \Expec \|X - (W^t)^{\top}(W^t (W^t)^{\top})^{-1}W^t X\|^2 \le \lambda_1 \Delta^t
\]
%
We prove the local exponential convergence of Matrix Krasulina's method measured by $\Delta^t$. Our main contribution is summarized by the following theorem.
\begin{thm}[Exponential convergence with constant learning rate]
\label{thm:main}
Suppose assumption Eq.~\eqref{eq:assumption} holds.
Suppose the initial estimate $W^o\in \mathbb{R}^{k^{\prime}\times d}$ ($k^{\prime}\ge k$) in Algorithm \ref{algo:kpca} satisfies that, for some $\tau \in (0, 1)$, 
\vspace{-0.2cm}
\[
tr(U^*P^o)\ge k-\frac{1-\tau}{2} \, ,
\vspace{-0.2cm}
\]
Suppose for any $\delta>0$, we choose a constant learning rate $\eta^t=\eta$ such that
%
\[
\eta \le \min \bigg\{\frac{\sqrt{2}-1}{b}, \frac{\lambda_k\tau}{\lambda_1 b (k+3)}, 
\frac{2\lambda_k\tau}{\frac{16}{1-\tau} \ln \frac{1}{\delta}(b+\|\Sigma^*\|_F)^2 + b(k+1)\lambda_1}
 \bigg\}
\]
Then there exists event $\mathcal{G}_t$ such that
$
\proba{\mathcal{G}_t} \ge 1 - \delta \, ,
$
and 
\[
\expec{\Delta^t | \mathcal{G}_{t}}
\le
\frac{1}{1-\delta} \exp{- t \eta \tau \lambda_k}
\]
\end{thm}
%
%
From Theorem \ref{thm:main}, we observe that
(a). The convergence rate of Algorithm \ref{algo:kpca} on strictly low-rank data does not depend on the data dimension $d$, but only on the intrinsic dimension $k$. This is verified by our experiments (see Sec.~\ref{sec:experiments}).
(b). We see that the learning rate should be of order $O(\frac{1}{k\lambda_1})$: Empirically, we found that setting $\eta$ to be roughly $\frac{1}{10\lambda_1}$ gives us the best convergence result. Note, however, this learning rate setup is not practical since it requires knowledge of eigenvalues.

\paragraph{Comparison between Theorem \ref{thm:main} and \citet[Theorem 1]{Shamir:general-oja:2016}}
(1).
The result in \citet{Shamir:general-oja:2016} does not rely on the low-rank assumption of $\Sigma^*$. Since the variance of update in Oja's method is not naturally decaying, they use VR technique inspired by \citet{svrg} to reduce the variance of the algorithm's iterate, which is computationally heavy: the block version of VR-PCA converges at rate $O(\exp{-CT})$, where $T$ denotes the number of data passes. 
(2).
Our result has a similar learning rate dependence on the data norm bound $b$ as that of  \citet[Theorem 1]{Shamir:general-oja:2016}.
(3).
The initialization requirement in Theorem \ref{thm:main} is comparable to \citet[Theorem 1]{Shamir:general-oja:2016}; we note that the factor $1/2$ in $\frac{1-\tau}{2}$ in our requirement is not strictly necessary in our analysis, and can be set arbitrarily close to 1.
(4).
Conditioning on the event of successful convergence, their exponential convergence rate result holds deterministically, whereas our convergence rate guarantee holds in expectation.

%


\subsection{Related Works}
Theoretical guarantees of stochastic optimization traditionally require convexity \citep{SCO}. 
However, many modern machine learning problems, especially those arising from deep learning and unsupervised learning, are non-convex; PCA is one of them: The objective in \eqref{eq:pca-population} is non-convex in $W$.
Despite this, a series of recent theoretical works have proven stochastic optimization to be effective for PCA, mostly variants of Oja's method \citep{allen-zhu:stochastic-pca:2017, Shamir:gapfree-pca:2016, Shamir:general-oja:2016, Shamir:exp-pca:2015, desa:sgd-for-matrix, hardt:noisy-power-method:2014, balsubramani13}.

Krasulina's method \citep{krasulina1969} was much less studied than Oja's method; a notable exception is the work of \citet{balsubramani13}, which proved an expected $O(1/t)$ rate for both Oja's and Krasulina's algorithm for $1$-PCA. 

There were very few theoretical analysis of stochastic $k$-PCA algorithms with $k>1$, with the exception of \citet{allen-zhu:stochastic-pca:2017, Shamir:general-oja:2016, pmlr-v49-balcan16a, pmlr-v51-li16b}. 
All had focused on variants of Oja's algorithm, among which
\citet{Shamir:general-oja:2016} was the only previous work, to the best of our knowledge, that provided a local exponential convergence rate guarantee of Oja's algorithm for $k\ge 1$. Their result holds for general data distribution, but their variant of Oja's algorithm, VR-PCA, requires several full passes over the datasets, and thus not fully online.

\paragraph{Open questions} 
In light of our result and related works, we have two open questions:
(1). While our analysis has focused on analyzing Algorithm \ref{algo:kpca} with a constant learning rate on low-rank data, we believe it can be easily adapted to show that with a $c/t$ (for some constant $c>0$) learning rate, the algorithm achieves $O(1/t)$ convergence on \textit{any datasets}. Note for the case $k^{\prime}=1$, the linear convergence rate of Algorithm \ref{algo:kpca} (original Krasulina's method) is already proved by \citet{balsubramani13}.
(2). Many real-world datasets are not strictly low-rank, but \textit{effectively low-rank} (see, for example, Figure \ref{fig:spectrum-vgg}): 
Informally, we say a dataset is effectively low-rank
%
\textbf{if there exists $k\ll d$ such that
$
\frac{\sum_{i> k}\lambda_i}{\sum_{j\le k}\lambda_j}
\text{~is small} \, ,
$
}
We conjecture that our analysis can be adapted to show theoretical guarantee of Algorithm \ref{algo:kpca} on effectively low-rank datasets as well. In Section \ref{sec:experiments}, our empirical results support this conjecture. 
Formally characterizing the dependence of convergence rate on the ``effective low-rankness'' of a dataset can provide a smooth transition between the worst-case lower bound in \citet{Vu:minimax-subspace:2013} and our result in Theorem \ref{thm:main}.

\section{Sketch of analysis}
In this section, we provide an overview of our analysis and lay out the proofs that lead to Theorem \ref{thm:main} (the complete proofs are deferred to the Appendix).
On a high level, our analysis is done in the following steps: 
\paragraph{Section~\ref{sec:conditional_martingale}} 
We show that if the algorithm's iterates, $W^t$, stay inside the basin of attraction, which we formally define as event $\mathcal{G}_t$,
\[
\mathcal{G}_t \Defeq \{\Delta^i \le 1-\tau, \forall i\le t\} \, ,
\] 
then a function of random variables $\Delta^t$ forms a supermartingale.
\paragraph{Section~\ref{sec:bound_bad_event}} 
We show that provided a good initialization, it is likely that \textit{the algorithm's outputs $W^1, \dots, W^t$ stay inside the basin of attraction for every $t$}. 
\paragraph{Section~\ref{sec:iter_wise}} We show that at each iteration $t$, conditioning on $\mathcal{G}_t$, $\Delta^{t+1}\le \beta \Delta^t$ for some $\beta < 1$ if we set the learning rate $\eta^t$ appropriately.
\paragraph{Appendix~\ref{sec:proof_thm}}
Iteratively applying this recurrence relation leads to Theorem \ref{thm:main}.

\paragraph{Additional notations:}
Before proceeding to our analysis, we introduce some technical notations for stochastic processes:
Let $(\mathcal{F}_t)$ denote the natural filtration (collection of $\sigma$-algebras) associated to the stochastic process, that is, the data stream $(X^t)$. 
Then by the update rule of  Algorithm \ref{algo:kpca}, for any $t$, $W^t$, $P^t$, and $\Delta^t$ are all $\mathcal{F}_t$-measurable, and $\mathcal{G}_t\in \mathcal{F}_t$.

\subsection{A conditional supermartingale}
\label{sec:conditional_martingale}
Letting $M_i \Defeq \indic{\mathcal{G}_{i-1}}\exp{s\Delta^i}$, Lemma~\ref{lm:conditional} shows that
$(M_i)_{i\ge 1}$
forms a supermartingale.
%
\begin{lm}[Supermartingale construction]
\label{lm:conditional}
Suppose $\mathcal{G}_0$ holds. 
Let $C^t$ and $Z$ be as defined in Proposition \ref{prop:iter_wise}. 
Then for any $i \leq t$, and for any constant $s>0$,
\begin{eqnarray*}
\expec{
\indic{\mathcal{G}_i} \exp{s\Delta^{i+1}} | \mathcal{F}_i} \\
\le
\indic{\mathcal{G}_{i-1}}
\exp{s\Delta^i \big(1-2\eta^{i+1} \lambda_k \tau+ (\eta^{i+1})^2 C^{i+1}\lambda_1 \big)
+2s^2(\eta^{i+1})^2 |Z|^2} \, .
\end{eqnarray*}
\end{lm}
The proof of Lemma~\ref{lm:conditional} utilizes the iteration-wise convergence inequality in Prop.~\ref{prop:iter_wise} of Section~\ref{sec:iter_wise}.
\subsection{Bounding probability of bad event $\mathcal{G}_t^c$}
\label{sec:bound_bad_event}
Let $\mathcal{G}_0$ denote the good event happening upon initialization of Algorithm \ref{algo:kpca}. Observe that the good events form a nested sequence of subsets through time:
\[
\mathcal{G}_0 \supset \mathcal{G}_1 \supset\dots \mathcal{G}_t \supset\dots 
\vspace{-0.2cm}
\]
This implies that we can partition the bad event $\mathcal{G}_t^c$ into a union of individual bad events: 
\vspace{-0.2cm}
\[
\mathcal{G}_t^c = \cup_{i=1}^t \bigg(\mathcal{G}_{i-1}\setminus\mathcal{G}_{i} \bigg) \, ,
\]
The idea behind Proposition~\ref{prop:bad_event} is that, we first transform the union of events above into a maximal inequality over a suitable sequence of random variables, which form a supermartingale, and then we apply a type of martingale large-deviation inequality to upper bound $\proba{\mathcal{G}_t^c }$.
%
\begin{prop}[Bounding probability of bad event]
\label{prop:bad_event}
Suppose the initialization condition in Theorem \ref{thm:main} holds. 
For any $\delta>0$, $t\ge 1$, and $i\le t$, 
%
%
if the learning rate $\eta^i$ is set such that 
\begin{eqnarray*}
\eta^{i} 
\le 
\min
\bigg\{
\frac{2\lambda_k\tau}{(\frac{16}{1-\tau} \ln \frac{1}{\delta}(b+\|\Sigma^*\|_F)^2 + b(k+1)\lambda_1)},
\frac{\sqrt{2}-1}{b}
\bigg\}
\, ,
\end{eqnarray*}
Then 
$
\proba{\mathcal{G}_{t}^c}
\leq \delta \, .
$
\end{prop}
\begin{proof}[Proof Sketch]
For $i>1$, we first consider the individual events:
\begin{eqnarray*}
\mathcal{G}_{i-1}\setminus \mathcal{G}_{i}
=
\mathcal{G}_{i-1} \cap  \mathcal{G}_{i}^c
=
\{
\forall j < i,~
\Delta^j
\le
1-\tau 
\}
\cap
\{
\Delta^i
>
1-\tau 
\}
\end{eqnarray*}
For any strictly increasing positive measurable function $g$, the above is equivalent to
\[
\mathcal{G}_{i-1}\setminus \mathcal{G}_{i}
=
\{
g(\Delta^i)
>
g(1-\tau)
\text{~~and~~}
\forall j < i,~
g(\Delta^j)
\le
g(1-\tau)
\}
\]
Since event $\mathcal{G}_{i-1}$ occurs is equivalent to $\{\indic{\mathcal{G}_{i-1}}=1\}$, we can write
\[
\mathcal{G}_{i-1}\setminus \mathcal{G}_{i}
=
\{
g(\Delta^i)
>
g(1-\tau)
\text{~~and~~}
\forall j < i,~
g(\Delta^j)
\le
g(1-\tau),
\text{~~and~~}
\indic{\mathcal{G}_{i-1}}=1
\}
\]
Additionally, since for any $j^{\prime}<j$, $\mathcal{G}_{j^{\prime}} \supset \mathcal{G}_{j}$, that is,
$\{\indic{\mathcal{G}_{j}}=1\}$ implies $\{\indic{\mathcal{G}_{j^{\prime}}}=1\}$, we have
\begin{eqnarray*}
\mathcal{G}_{i-1}\setminus \mathcal{G}_{i} \\
= 
\{g(\Delta^i)>g(1-\tau)
\text{~and~}
\forall j < i,~
g(\Delta^j)
\le
g(1-\tau), 
\indic{\mathcal{G}_{i-1}}=1,
\indic{\mathcal{G}_{j^{\prime}}}=1, \forall j^{\prime}<i-1
\} \\
=
\{\indic{\mathcal{G}_{i-1}} g(\Delta^i)>g(1-\tau) \text{~and~} \forall j<i, \indic{\mathcal{G}_{j-1}}g(\Delta^j)  \le g(1-\tau), \text{~and~}\indic{\mathcal{G}_{j}}=1\} \\
\subset
\{\indic{\mathcal{G}_{i-1}} g(\Delta^i)>g(1-\tau) \text{~and~}   \forall j<i, \indic{\mathcal{G}_{j-1}}g(\Delta^j)  \le g(1-\tau)\} 
\end{eqnarray*}
So the union of the terms $\mathcal{G}_{i-1}\setminus \mathcal{G}_{i}$ can be upper bounded as
\begin{eqnarray*}
\cup_{i=1}^t \mathcal{G}_{i-1}\setminus \mathcal{G}_{i} \subset \\
\cup_{i=2}^t
\{\indic{\mathcal{G}_{i-1}} g(\Delta^i)>g(1-\tau) ,  \indic{\mathcal{G}_{j-1}}g(\Delta^j)  \le g(1-\tau), \forall 1\le j<i\}  \\
\cup
\{\indic{\mathcal{G}_{0}} g(\Delta^1)>g(1-\tau) \}
\end{eqnarray*}
Observe that the event above can also be written as
\[
\{\sup_{1\le i\le t} \indic{\mathcal{G}_{i-1}} g(\Delta^i)
>
g(1-\tau)\} \, .
\]
We upper bound the probability of the event above
by applying a variant of Doob's inequality.
To achieve this, the key step is to find a suitable function $g$ such that the sequence
\[
\indic{\mathcal{G}_{0}} g(\Delta^1), \indic{\mathcal{G}_{1}} g(\Delta^2), \dots, \indic{\mathcal{G}_{i-1}} g(\Delta^i), \dots
\]
forms a supermartingale.
Via Lemma \ref{lm:conditional}, we show that if we choose $g(x)\Defeq \exp{sx}$ for any constant $s>0$, then
\begin{eqnarray}
\label{eqn:supermartingale}
\expec{\indic{\mathcal{G}_i} \exp{s\Delta^{i+1}} | \mathcal{F}_i}
\le
\indic{\mathcal{G}_{i-1}}
\exp{s\Delta^i} \, ,
\end{eqnarray}
provided we choose the learning rate in Algorithm \ref{algo:kpca} appropriately.
Then a version of Doob's inequality for supermartingale \citep[p. 231]{balsubramani13, Durrett11probability:theory} implies that
\begin{eqnarray*}
\proba{\sup_i \indic{\mathcal{G}_{i-1}} \exp{s\Delta^i}
>
\exp{s(1-\tau)}} 
\le
\frac{\expec{\indic{\mathcal{G}_{0}} \exp{s\Delta^1}}}{\exp{s(1-\tau)}} \, ,
\end{eqnarray*}
Finally, bounding the expectation on the RHS using our assumption on the initialization condition finishes the proof.
\end{proof}
%
%
%
\subsection{Iteration-wise convergence result}
\label{sec:iter_wise}
\begin{prop}[Iteration-wise subspace improvement]
\label{prop:iter_wise}
At the $t+1$-th iteration of Algorithm \ref{algo:kpca}, 
the following holds:
\begin{itemize}
\item[(V1)]
Let
$
C^{t}\Defeq kb + 2\eta^{t}b^2 + (\eta^{t})^2b^3 \, .
$
Then
\begin{eqnarray*}
\expec{tr(U^* P^{t+1})|\mathcal{F}_t} 
\geq 
tr(U^* P^t)  
+
2\eta^{t+1} \lambda_k \Delta^t (1-\Delta^t)
-
(\eta^{t+1})^2 C^{t+1} \lambda_1\Delta^t
\end{eqnarray*}
\item[(V2)]
There exists a random variable $Z$, 
with
\[
\expec{Z | \mathcal{F}_t} = 0
\text{~~and~~}
|Z| 
\le
2 (b +  \|\Sigma^*\|_F) \sqrt{\Delta^t}
\]
such that
\begin{eqnarray*}
tr(U^* P^{t+1}) 
\geq 
tr(U^* P^{t})  
+
2\eta^{t+1} \lambda_k \Delta^t (1-\Delta^t) 
+ 2\eta^{t+1} Z
- (\eta^{t+1})^2 C^{t+1} \lambda_1\Delta^t
\end{eqnarray*}
\end{itemize}
\end{prop}

\begin{proof}[Proof Sketch]
By definition,
\begin{equation*}
tr (U^* P^{t+1})
=
tr (U^* (W^{t+1})^{\top} (W^{t+1} (W^{t+1})^{\top})^{-1} W^{t+1}) \, ,
\end{equation*}
where by the update rule of Algorithm \ref{algo:kpca}
\[
W^{t+1} = W^t + \eta^{t+1} s^{t+1} (r^{t+1})^{\top} \, .
\]
We first derive (V1); 
the proof sketch is as follows:
\begin{enumerate}
\item
Since the rows of $W^{t}$ are orthonormalized, one would expect that a small perturbation of this matrix, $W^{t+1}$, is also close to orthonormalized, and thus $W^{t+1}(W^{t+1})^{\top}$ should be close to an identity matrix. 
Lemma~\ref{lm:inverse_matrix_bound} shows this is indeed the case, offsetting by a small term $E$, which can be viewed as an error/excessive term:
\begin{lm}[Inverse matrix approximation]
\label{lm:inverse_matrix_bound}
Let $k^{\prime}$ be the number of rows in $W^t$. 
Suppose the rows of $W^t$ are orthonormal, that is,
$
W^t (W^t)^{\top} = I_{k^{\prime}} \, .
$
Then for
$
W^{t+1} = W^t + \eta^{t+1} s^{t+1} (r^{t+1})^{\top} \, ,
$
we have
\[
(W^{t+1} (W^{t+1})^{\top})^{-1} \succeq (1 - \lambda_1(E)) I_{k^{\prime}} \, ,
\]
where 
$\lambda_1(E)$ is the largest eigenvalue of some matrix $E$, and
$\lambda_1(E) = (\eta^{t+1})^2 \|r^{t+1}\|^2 \|s^{t+1}\|^2 \, .$
\end{lm}
This implies that
\begin{eqnarray*}
tr U^* (W^{t+1})^{\top} (W^{t+1} (W^{t+1})^{\top})^{-1} W^{t+1} \\
\geq
(1- (\eta^{t+1})^2 \|r^{t+1}\|^2 \|s^{t+1}\|^2) tr (U^*(W^{t+1} )^{\top} W^{t+1}) 
\end{eqnarray*}
\item
We continue to lower bound the conditional expectation of the last term in the previous inequality as
\begin{eqnarray*}
\expec{tr(U^* (W^{t+1})^{\top} W^{t+1})|\mathcal{F}_t} 
\geq 
tr(U^* P^t) 
+ 
2\eta^{t+1} tr (U^* P^t \Sigma^*(I_d-P^t))
\end{eqnarray*} 
\item
The last term in the inequality above, 
$
tr (U^* P^t \Sigma^*(I_d-P^t)) \, ,
$
controls the improvement in proximity between the estimated and the ground-truth subspaces.
In Lemma~\ref{lm:stationary_pts}, we lower bound it as a function of $\Delta^t$:
\begin{lm}[Characterization of stationary points] 
\label{lm:stationary_pts}
Let 
\[
\Gamma^t\Defeq tr (U^*P^t \Sigma^*(I_d-P^t)) \, ,
\]
Then the following holds:
\begin{enumerate}
\item
$
tr(U^* P^t) = tr(U^*) 
$ implies that
$
\Gamma^t = 0 \, .
$
\item
$
\Gamma^t 
\ge
\lambda_k \Delta^t (1-\Delta^t) \, .
$
\end{enumerate}
\end{lm}
\item
Finally, combining the results above, we obtain (V1) inequality in the statement of the proposition.
\end{enumerate}
(V2) inequality is derived similarly with the steps above, except that at step 2, instead of considering the conditional expectation of $tr(U^* (W^{t+1})^{\top} W^{t+1})$, we explicitly represent the zero-mean random variable $Z$ in the inequality.
\end{proof}
\section{Experiments}
\label{sec:experiments}
In this section, we present our empirical evaluation of Algorithm \ref{algo:kpca}. We first verified its performance on simulated low-rank data and effectively low-rank data, and then we evaluated its performance on two real-world effectively low-rank datasets.

\begin{figure*}[ht]
\centering
\begin{minipage}[b]{0.33\textwidth}
  \centering
  \includegraphics[width=\textwidth]{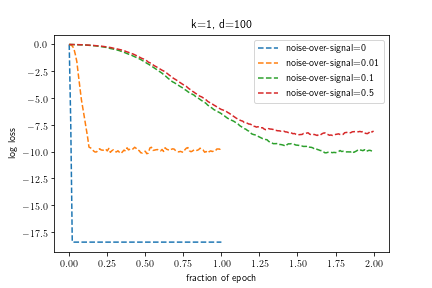}
  \caption*{$k=1,\, d=100$}
\end{minipage}%
\hfill
\begin{minipage}[b]{0.33\textwidth}
  \centering
  \includegraphics[width=\textwidth]{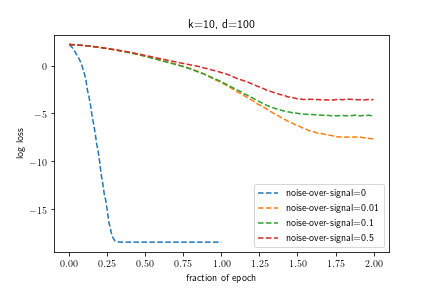}
  \caption*{$k=10,\, d=100$}
\end{minipage}
\hfill
\begin{minipage}[b]{0.33\textwidth}
  \centering
  \includegraphics[width=\textwidth]{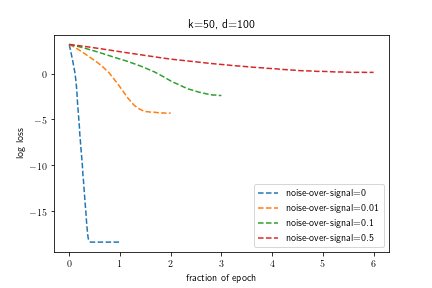}
  \caption*{$k=50,\, d=100$}
\end{minipage}
\begin{minipage}[b]{0.33\textwidth}
  \centering
  \includegraphics[width=\linewidth]{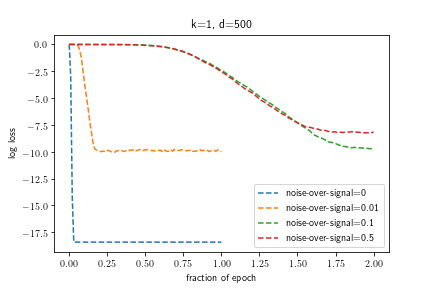}
  \caption*{$k=1,\, d=500$}
\end{minipage}%
\hfill
\begin{minipage}[b]{0.33\textwidth}
  \centering
  \includegraphics[width=\linewidth]{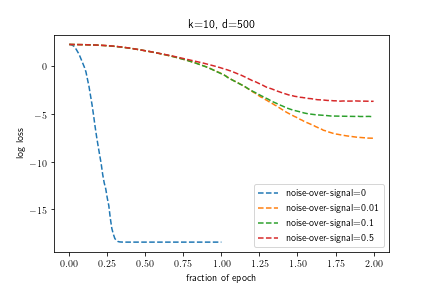}
  \caption*{$k=10,\, d=500$}
\end{minipage}
\hfill
\begin{minipage}[b]{0.33\textwidth}
  \centering
  \includegraphics[width=\linewidth]{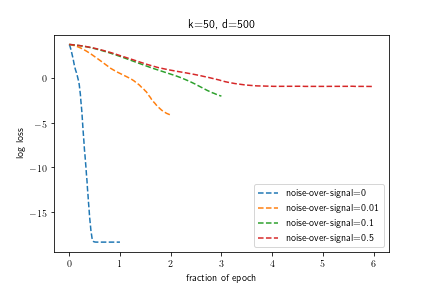}
  \caption*{$k=50,\, d=500$}
\end{minipage}
\caption{
log-convergence graph of Algorithm \ref{algo:kpca}:
$\ln(\Delta^t)$ vs $t$ at different levels of noise-over-signal ratio ($\frac{\sum_{i>k}\lambda_i}{\sum_{j\le k}\lambda_j}$)}
\label{fig:fig-sim}
\end{figure*}

\begin{figure*}[ht]
\begin{minipage}[t]{0.33\textwidth}
  \centering
  \includegraphics[width=\linewidth, scale=0.8]{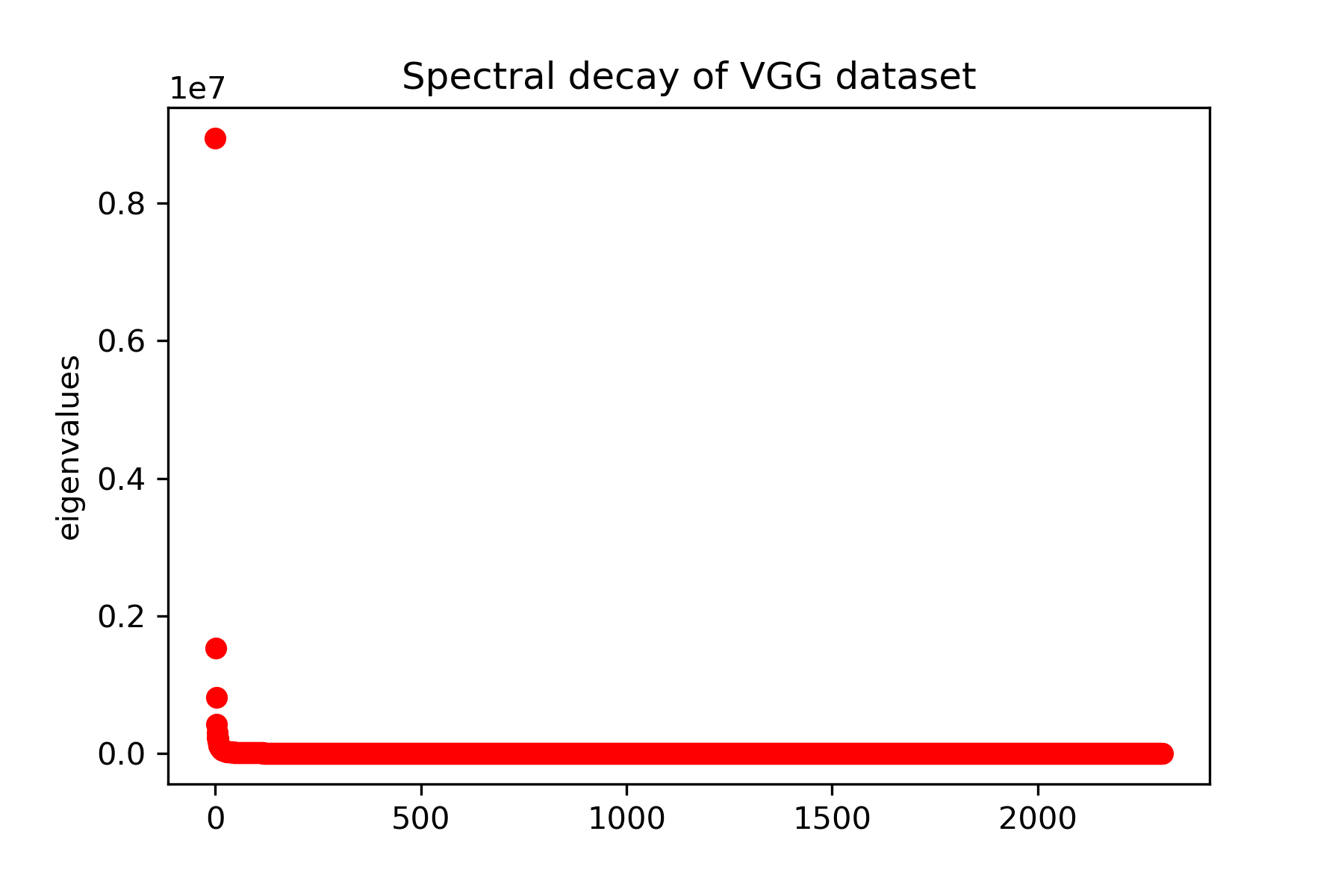}
  \caption{top 6 eigenvalues explains 80\% of the data variance.}
  \label{fig:spectrum-vgg}
\end{minipage}
\begin{minipage}[t]{0.33\textwidth}
  \centering
  \includegraphics[width=\linewidth]{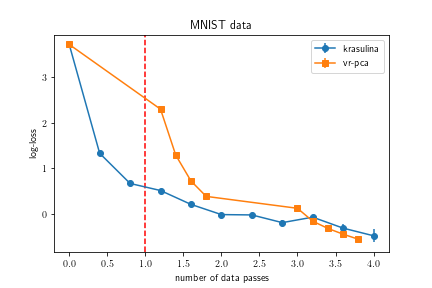}
 \caption*{MNIST ($d=784; k^{\prime}=44$)}
\end{minipage}%
\begin{minipage}[t]{0.33\textwidth}
  \centering
  \includegraphics[width=\linewidth]{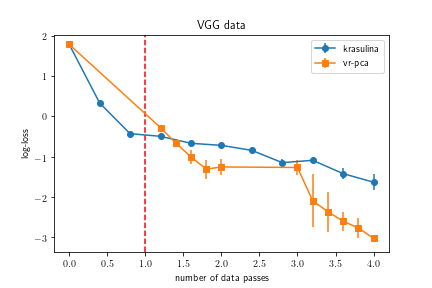}
  \caption*{VGG ($d=2304; k^{\prime}=6$); red vertical line marks a full pass over the dataset}
\end{minipage}
\label{fig:fig-real}
\end{figure*}
\vspace{-0.5cm}
\subsection{Simulations}
The low-rank data is generated as follows: we sample i.i.d. standard normal on the first $k$ coordinates of the $d$-dimensional data (the rest $d-k$ coordinates are zero), then we rotate all data using a random orthogonal matrix (unknown to the algorithm). 
\paragraph{Simulating effectively low-rank data} In practice, hardly any dataset is strictly low-rank but many datasets have sharply decaying spectra (recall Figure~\ref{fig:spectrum-vgg}).
Although our Theorem \ref{thm:main} is developed under a strict low-rankness assumption, here we empirically test the robustness of our convergence result when data is not strictly low rank but only effectively low rank. Let $\lambda_1,\ge\dots\ge\lambda_d\ge 0$ be the spectrum of a covariance matrix. For a fixed $k\in [d]$, we let
\textit{noise-over-signal} $\Defeq \frac{\sum_{i> k}\lambda_i}{\sum_{j\le k}\lambda_j}\, .$
The \textit{noise-over-signal} ratio intuitively measures how ``close'' the matrix is to a rank-$k$ matrix: The smaller the number is, the shaper the spectral decay; when the ratio equals zero, the matrix is of rank at most $k$.
In our simulated data, we perturb the spectrum of a strictly rank-$k$ covariance matrix and generate data with full-rank covariance matrices at the following noise-over-signal ratios, 
$\{0, 0.01, 0.1, 0.5\}$.

\paragraph{Results}
Figure \ref{fig:fig-sim} shows the \textit{log-convergence graph} of Algorithm \ref{algo:kpca} on our simulated data: In contrast to the local initialization condition in Theorem \ref{thm:main}, we initialized Algorithm \ref{algo:kpca} with a random matrix $W^o$ and ran it for one or a few epochs, each consists of $5000$ iterations. 
(1). We verified that, on strictly low rank data (noise-over-signal$=0$), the algorithm indeed has an exponentially convergence rate (linear in log-error);
(2). As we increase the noise-over-signal ratio, the convergence rate gradually becomes slower;
(3). The convergence rate is not affected by the actual data dimension $d$, but only by the intrinsic dimension $k$, as predicted by Theorem \ref{thm:main}.

\subsection{Real effectively low-rank datasets}
We take a step further to test the performance of Algorithm \ref{algo:kpca} on two real-world datasets: 
VGG \citep{dataset-vgg} is a dataset of 10806 image files from 2622 distinct celebrities crawled from the web, with $d=2304$.
For MNIST \citep{dataset-mnist}, we use the 60000 training examples of digit pixel images, with $d=784$.
Both datasets are full-rank, but we choose $k^{\prime}$ such that the noise-over-signal ratio at $k^{\prime}$ is 0.25; that is, the top $k^{\prime}$ eigenvalues explain 80\% of data variance.
We compare Algorithm \ref{algo:kpca} against the exponentially convergent VR-PCA: we initialize the algorithms with the same random matrix and we train (and repeated for $5$ times) using the best constant learning rate we found empirically for each algorithm.
We see that Algorithm \ref{algo:kpca} retains fast convergence even if the datasets are not strictly low rank, and that it has a clear advantage over VR-PCA before the iteration reaches a full pass; indeed, VR-PCA requires a full-pass over the dataset before its first iterate.

\bibliography{mybib}
\bibliographystyle{icml2019}

\appendix
\newpage
\section{Proofs for Proposition~\ref{prop:bad_event} }
\begin{proof}[Proof of Proposition~\ref{prop:bad_event}]
Recall definition of $\mathcal{G}_{t}$,
$
\mathcal{G}_{t}\Defeq \{\Delta^i \le 1-\tau, \forall i\le t\} \, .
$
We partition its complement as
$
\mathcal{G}_{t}^c
=
\cup_{i=1}^t \mathcal{G}_{i-1}\setminus \mathcal{G}_{i} \, .
$
For $i>1$, we first consider the individual events:
\begin{eqnarray*}
\mathcal{G}_{i-1}\setminus \mathcal{G}_{i}
=
\mathcal{G}_{i-1} \cap  \mathcal{G}_{i}^c
=
\{
\Delta^i
>
1-\tau 
\}
\cap 
\{
\forall j < i,~
\Delta^j
\le
1-\tau 
\}
\end{eqnarray*}
For any strictly increasing positive measurable function $g$, the above is equivalent to
\[
\mathcal{G}_{i-1}\setminus \mathcal{G}_{i}
=
\{
g(\Delta^i)
>
g(1-\tau)
\text{~~and~~}
\forall j < i,~
g(\Delta^j)
\le
g(1-\tau)
\}
\]
Since event $\mathcal{G}_{i-1}$ occurs is equivalent to $\{\indic{\mathcal{G}_{i-1}}=1\}$, we can write
\[
\mathcal{G}_{i-1}\setminus \mathcal{G}_{i}
=
\{
g(\Delta^i)
>
g(1-\tau)
\text{~~and~~}
\forall j < i,~
g(\Delta^j)
\le
g(1-\tau),
\text{~~and~~}
\indic{\mathcal{G}_{i-1}}=1
\}
\]
Additionally, since for any $j^{\prime}<j$, $\mathcal{G}_{j^{\prime}} \supset \mathcal{G}_{j}$, that is,
$\{\indic{\mathcal{G}_{j}}=1\}$ implies $\{\indic{\mathcal{G}_{j^{\prime}}}=1\}$, we have
\begin{eqnarray*}
\mathcal{G}_{i-1}\setminus \mathcal{G}_{i} \\
= 
\{g(\Delta^i)>g(1-\tau)
\text{~and~}
\forall j < i,~
g(\Delta^j)
\le
g(1-\tau), 
\indic{\mathcal{G}_{i-1}}=1,
\indic{\mathcal{G}_{j^{\prime}}}=1, \forall j^{\prime}<i-1
\} \\
=
\{\indic{\mathcal{G}_{i-1}} g(\Delta^i)>g(1-\tau) \text{~and~} \forall j<i, \indic{\mathcal{G}_{j-1}}g(\Delta^j)  \le g(1-\tau), \text{~and~}\indic{\mathcal{G}_{j}}=1\} \\
\subset
\{\indic{\mathcal{G}_{i-1}} g(\Delta^i)>g(1-\tau) \text{~and~}   \forall j<i, \indic{\mathcal{G}_{j-1}}g(\Delta^j)  \le g(1-\tau)\} 
\end{eqnarray*}
So the union of the terms $\mathcal{G}_{i-1}\setminus \mathcal{G}_{i}$ can be upper bounded as
\begin{eqnarray*}
\cup_{i=1}^t \mathcal{G}_{i-1}\setminus \mathcal{G}_{i} \subset \\
\cup_{i=2}^t
\{\indic{\mathcal{G}_{i-1}} g(\Delta^i)>g(1-\tau) ,  \indic{\mathcal{G}_{j-1}}g(\Delta^j)  \le g(1-\tau), \forall 1\le j<i\}  \\
\cup
\{\indic{\mathcal{G}_{0}} g(\Delta^1)>g(1-\tau) \}
\end{eqnarray*}
Observe that the event above can also be written as
\[
\{\sup_{1\le i\le t} \indic{\mathcal{G}_{i-1}} g(\Delta^i)
>
g(1-\tau)\} \, .
\]
Now we upper bound 
$
\proba{\{\sup_{1\le i\le t} \indic{\mathcal{G}_{i-1}} g(\Delta^i)
>
g(1-\tau)\}}
$
by applying a martingale large deviation inequality.
To achieve this, the key step is to find a suitable function $g$ such that the stochastic process
\[
\indic{\mathcal{G}_{0}} g(\Delta^1), \indic{\mathcal{G}_{1}} g(\Delta^2), \dots, \indic{\mathcal{G}_{i-1}} g(\Delta^i), \dots
\]
is a supermartingale. In this proof, we choose 
$g:\mathbb{R}\rightarrow\mathbb{R}_{>0}$ to be $g(x) = \exp{sx}$ for $s=\frac{2}{1-\tau} \ln \frac{1}{\delta}$.
\newline
By Lemma \ref{lm:conditional},
\begin{eqnarray*}
\expec{
\indic{\mathcal{G}_i} \exp{s\Delta^{i+1}} | \mathcal{F}_i} \\
\le
\indic{\mathcal{G}_{i-1}}
\exp{s\Delta^i \big(1-2\eta^{i+1} \lambda_k \tau+ (\eta^{i+1})^2 C^{i+1}\lambda_1 \big)
+2s^2(\eta^{i+1})^2 |Z|^2} \, \\
\le
\indic{\mathcal{G}_{i-1}}\exp{
s\Delta^i \big(1-2\eta^{i+1} \lambda_k \tau+ (\eta^{i+1})^2 C^{i+1}\lambda_1 \big)
}
\exp{s^2(\eta^{i+1})^2 8(b + \|\Sigma^*\|_F)^2\Delta^i} \\
=
\indic{\mathcal{G}_{i-1}}
\exp{
s\Delta^i \bigg(
1-2\eta^{i+1}\lambda_k\tau + (\eta^{i+1})^2 C^{i+1}\lambda_1 
+ s(\eta^{i+1})^2 8(b + \|\Sigma^*\|_F)^2
\bigg)
} 
\end{eqnarray*}
Since we choose the learning rate in Algorithm \ref{algo:kpca} such that
\begin{eqnarray}
\label{eqn:1-1}
\eta^{i+1} < 
\frac{2\lambda_k\tau}{b(k+1)\lambda_1 + \frac{16}{1-\tau}\ln\frac{1}{\delta} (b + \|\Sigma^*\|_F)^2}
=
\frac{2\lambda_k\tau}{b(k+1)\lambda_1 + 8 s (b + \|\Sigma^*\|_F)^2} \, .
\end{eqnarray}
And since $\eta^{i+1} \le \frac{\sqrt{2}-1}{b}$, it can be seen that
\begin{eqnarray}
\label{eqn:1-2}
C^{i+1}=kb+2\eta^{i+1}b^2+(\eta^{i+1})^2 b^3 \le b(k+1)
\end{eqnarray}
Combining Eq~\eqref{eqn:1-1} and ~\eqref{eqn:1-2}, we get
\[
-2\eta^{i+1}\lambda_k\tau + (\eta^{i+1})^2 C^{i+1}\lambda_1 
+ s(\eta^{i+1})^2 8(b + \|\Sigma^*\|_F)^2
\le 0
\]
Therefore,
\[
\expec{\indic{\mathcal{G}_i} \exp{s\Delta^{i+1}} | \mathcal{F}_i}
\le
\indic{\mathcal{G}_{i-1}}
\exp{
s\Delta^i
}
\]
Thus, letting $M_i = \indic{\mathcal{G}_{i-1}}\exp{s\Delta^i}$, 
$(M_i)_{i\ge 1}$
forms a supermartingale. A version of Doob's inequality for supermartingale \cite[p. 231]{Durrett11probability:theory} implies that
\begin{eqnarray*}
\proba{\mathcal{G}_{t}^c}
=
\proba{\cup_{i=1}^t \mathcal{G}_{i-1}\setminus \mathcal{G}_{i}} \\
\le
\proba{\sup_{i\ge1} \indic{\mathcal{G}_{i-1}} \exp{s\Delta^i}
>
\exp{s(1-\tau)}}
=
\proba{\sup_{i\ge1} M_i > \exp{s(1-\tau)}}
 \\
\le
\frac{\expec{M_1}}{\exp{s(1-\tau)}}
=
\frac{\expec{\indic{\mathcal{G}_{0}} \exp{s\Delta^1}}}{\exp{s(1-\tau)}}
\end{eqnarray*}
We bound the expectation as follows: By Inequality~\ref{eqn:1} of Lemma \ref{lm:conditional},
\[
\exp{s\Delta^1}\indic{\mathcal{G}_0}
\le
\exp{
s \bigg( \Delta^0 (1-2\eta^{1} \lambda_k(1-\Delta^0))
-
2\eta^{1} Z
+ (\eta^{1})^2 C^{1}\lambda_1 \Delta^0 \bigg)
} \indic{\mathcal{G}_0}
\]
Taking expectation on both sides,
\begin{eqnarray*}
\expec{\indic{\mathcal{G}_{0}} \exp{s\Delta^1}} \\
\le 
\exp{
s \bigg( \Delta^0 (1-2\eta^{1} \lambda_k(1-\Delta^0))
+ (\eta^{1})^2 C^{1}\lambda_1 \Delta^0 \bigg)
}
\expec{
\exp{s(-2\eta^1Z)}
} \\
\le
\exp{
s \bigg( \Delta^0 (1-2\eta^{1} \lambda_k(1-\Delta^0))
+ (\eta^{1})^2 C^{1}\lambda_1 \Delta^0 \bigg)
}
\exp{2 s^2(\eta^{1})^2 |Z|^2}  \\
\le
\exp{
s\Delta^0 \bigg(
1-2\eta^{1}\lambda_k\tau + (\eta^{1})^2 C^{1}\lambda_1 
+ s(\eta^{1})^2 8(b + \|\Sigma^*\|_F)^2
\bigg)
} \\
\le
\exp{
s \frac{1-\tau}{2} \bigg(
1-2\eta^{1}\lambda_k\tau + (\eta^{1})^2 C^{1}\lambda_1 
+ s(\eta^{1})^2 8(b + \|\Sigma^*\|_F)^2
\bigg)
} \\
\le
\exp{
s \frac{1-\tau}{2} 
}
\end{eqnarray*}
where the second inequality holds by Hoeffding's lemma (using the same argument as in Lemma \ref{lm:conditional}), and the third and fourth inequality is by the fact that $\Delta^0\le \frac{1-\tau}{2}$ holds by our assumption.
Finally,
\[
\frac{\expec{\indic{\mathcal{G}_{0}} \exp{s\Delta^1}}}{\exp{s(1-\tau)}}
\le
\exp{
-s (1-\tau) /2
}
\le
\delta \, ,
\]
since we set $s = \frac{2}{1-\tau} \ln \frac{1}{\delta}$.
\end{proof}
\subsection{Auxiliary lemma for Proposition~\ref{prop:bad_event}}
\begin{proof}[Proof of Lemma \ref{lm:conditional}]
By V2 of Proposition \ref{prop:iter_wise}, for $\Sigma^*$ with rank $k$,
\begin{eqnarray*}
tr(U^* P^{i+1}) 
\geq
tr(U^* P^{i})  \\
+
2\eta^{i+1}\sum_{\ell=1}^k \lambda_{\ell} (1-u_{\ell}^{\top}P^i u_{\ell})
(u_{\ell}^{\top}P^i u_{\ell} - \sum_{m\ne \ell} [1-u_m^{\top}P^i u_m]) 
+ 2\eta^{i+1} Z \\
- (\eta^{i+1})^2 C^{i+1} tr(\Sigma^* - \Sigma^* P^i)
\end{eqnarray*}
From this, we can derive
\begin{eqnarray*}
\Delta^{i+1} 
\le
\Delta^i - 2\eta^{i+1}\sum_{\ell=1}^k \lambda_{\ell} (1-u_{\ell}^{\top}P^tu_{\ell})
(1-\Delta^i)  
- 2\eta^{i+1} Z
+ (\eta^{i+1})^2 C^{i+1}\lambda_1 \Delta^i \nonumber\\
\le
\Delta^i - 2\eta^{i+1} \lambda_k tr(U^*-U^*P^i)
(1-\Delta^i)  
- 2\eta^{i+1} Z
+ (\eta^{i+1})^2 C^{i+1}\lambda_1 \Delta^i \nonumber\\
=
\Delta^i - 2\eta^{i+1} \lambda_k \Delta^i
(1-\Delta^i)  
- 2\eta^{i+1} Z
+ (\eta^{i+1})^2 C^{i+1}\lambda_1 \Delta^i \nonumber\\
=
\Delta^i (1-2\eta^{i+1} \lambda_k(1-\Delta^i))
-
2\eta^{i+1} Z
+ (\eta^{i+1})^2 C^{i+1}\lambda_1 \Delta^i \\
\end{eqnarray*}
This implies that for any $s>0$,
\[
\exp{s\Delta^{i+1}}
\le
\exp{
s \bigg( \Delta^i (1-2\eta^{i+1} \lambda_k(1-\Delta^i))
-
2\eta^{i+1} Z
+ (\eta^{i+1})^2 C^{i+1}\lambda_1 \Delta^i \bigg)
}
\]
Multiplying both sides of the inequality by $\indic{\mathcal{G}_i}$, we get
\begin{eqnarray}
\label{eqn:1}
\exp{s\Delta^{i+1}} \indic{\mathcal{G}_i} \nonumber\\
\le
\exp{
s \bigg( \Delta^i (1-2\eta^{i+1} \lambda_k(1-\Delta^i))
-
2\eta^{i+1} Z
+ (\eta^{i+1})^2 C^{i+1}\lambda_1 \Delta^i \bigg)
} \indic{\mathcal{G}_i}  \,
\end{eqnarray}
We can further upper bound the RHS of Inequality~\eqref{eqn:1} above as
\begin{eqnarray*}
\exp{s\bigg( \Delta^i (1-2\eta^{i+1} \lambda_k(1-\Delta^i))
-
2\eta^{i+1} Z
+ (\eta^{i+1})^2 C^{i+1}\lambda_1 \Delta^i \bigg)} \indic{\mathcal{G}_i} \\
\le
\exp{s\bigg( \Delta^i (1-2\eta^{i+1} \lambda_k \tau)
-
2\eta^{i+1} Z
+ (\eta^{i+1})^2 C^{i+1}\lambda_1 \Delta^i \bigg)} \indic{\mathcal{G}_i} \\
\le
\exp{s\bigg( \Delta^i (1-2\eta^{i+1} \lambda_k \tau)
-
2\eta^{i+1} Z
+ (\eta^{i+1})^2 C^{i+1}\lambda_1 \Delta^i \bigg)} \indic{\mathcal{G}_{i-1}} \\
\le
\indic{\mathcal{G}_{i-1}}\exp{s \bigg( \Delta^i (1-2\eta^{i+1} \lambda_k \tau)+ (\eta^{i+1})^2 C^{i+1}\lambda_1 \Delta^i \bigg)}
\exp{s \big(-2\eta^{i+1} Z \big)}
\end{eqnarray*}
The first inequality is due to the fact that ``$\{\indic{\mathcal{G}_i}=1\}$ implies 
$
\{
\Delta^i \le 1-\tau
\}
$''
and the second inequality holds since 
$\mathcal{G}_i \subset \mathcal{G}_{i-1} \, .$
Incorporating this bound into inequality~\eqref{eqn:1} and taking conditional expectation w.r.t. $\mathcal{F}_i$ on both sides, we get
\begin{eqnarray*}
\indic{\mathcal{G}_i} \expec{\exp{s\Delta^{i+1}} | \mathcal{F}_i} 
=
\expec{\exp{s\Delta^{i+1}} \indic{\mathcal{G}_i} | \mathcal{F}_i} 
\\
\le
\indic{\mathcal{G}_{i-1}}\exp{s \bigg( \Delta^i (1-2\eta^{i+1} \lambda_k \tau)+ (\eta^{i+1})^2 C^{i+1}\lambda_1 \Delta^i \bigg)}
\expec{\exp{s \big(-2\eta^{i+1} Z \big)} | \mathcal{F}_i}
\end{eqnarray*}
Now we upper bound 
$
\expec{\exp{s \big(-2\eta^{i+1} Z \big)} | \mathcal{F}_i} \,:
$
Since 
\[
- 2 \eta^{i+1} |Z|
\le
2 \eta^{i+1} (-Z)
\le 
2 \eta^{i+1} |Z| \, ,
\]
and
\[
\expec  {2 s \eta^{i+1} (-Z) |\mathcal{F}_i} = \expec  {2 s \eta^{i+1} Z |\mathcal{F}_i}= 0 \, ,
\]
by Hoeffding's lemma
\begin{eqnarray*}
\expec{ 
\exp{2 s \eta^{i+1} (-Z) | \mathcal{F}_i}
}
\le 
\exp{\frac{s^2 (4 \eta^{i+1} |Z|)^2}{8}}
=
\exp{2s^2(\eta^{i+1})^2 |Z|^2} \, .
\end{eqnarray*}
Combining this with the previous bound, we get
\begin{eqnarray*}
\indic{\mathcal{G}_i} \expec{\exp{s\Delta^{i+1}} | \mathcal{F}_i}  \\
\le
\indic{\mathcal{G}_{i-1}}\exp{s \bigg( \Delta^i (1-2\eta^{i+1} \lambda_k \tau)+ (\eta^{i+1})^2 C^{i+1}\lambda_1 \Delta^i \bigg)}
\exp{2 s^2(\eta^{i+1})^2 |Z|^2}  
\end{eqnarray*}
\end{proof}
\section{Proofs for Proposition~\ref{prop:iter_wise}}
\begin{proof}[Proof of Proposition \ref{prop:iter_wise}]
We consider
\[
\expec{tr U^* P^{t+1} \big | \mathcal{F}_t }
=
\expec{tr U^* (W^{t+1})^{\top} (W^{t+1} (W^{t+1})^{\top})^{-1} W^{t+1} \big | \mathcal{F}_t } \, ,
\]
Since $U^*$ is positive semidefinite, we can write it as
$U^* = ((U^*)^{1/2})^2$. 
By the proof of Lemma~\ref{lm:inverse_matrix_bound}, 
\[
(W^{t+1} (W^{t+1})^{\top})^{-1} 
\succeq
(1-(\eta^{t+1})^2\|r^{t+1}\|^2 \|s^{t+1}\|^2)I_{k^{\prime}}
\, 
\] 
Letting $V\Defeq W^{t+1} (U^*)^{1/2}$, this implies that
\[
V^{\top}
\Big [W^{t+1} (W^{t+1})^{\top})^{-1}-(1-(\eta^{t+1})^2\|r^{t+1}\|^2 \|s^{t+1}\|^2)I_{k^{\prime}}
\Big ]
V
\succeq
0
\]
That is, the matrix on the left-hand-side above is positive semi-definite.
Since trace of a positive semi-definite matrix is non-negative, we have
\[
tr (V^{\top} W^{t+1} (W^{t+1})^{\top})^{-1} V)
\ge
tr (V^{\top} (1-(\eta^{t+1})^2\|r^{t+1}\|^2 \|s^{t+1}\|^2)V)
\]
By commutative property of trace, we further get
\begin{eqnarray*}
tr(U^*(W^{t+1})^{\top} [W^{t+1}(W^{t+1})^{\top}]^{-1}W^{t+1})
=
tr (V^{\top} W^{t+1} (W^{t+1})^{\top})^{-1} V) \\
\ge
tr (V^{\top} (1-(\eta^{t+1})^2\|r^{t+1}\|^2 \|s^{t+1}\|^2) V)  \\
=
(1- (\eta^{t+1})^2 \|r^{t+1}\|^2 \|s^{t+1}\|^2) tr (U^*(W^{t+1} )^{\top} W^{t+1})
\end{eqnarray*}
Taking expectation on both sides, we get
\[
\expec{tr U^* P^{t+1} \big | \mathcal{F}_t}
\ge
(1- (\eta^{t+1})^2 \|r^{t+1}\|^2 \|s^{t+1}\|^2)
\expec{ tr (U^*(W^{t+1} )^{\top} W^{t+1}) \big | \mathcal{F}_t }
\]
Now we in turn lower bound 
$
\expec{ tr [U^*(W^{t+1} )^{\top} W^{t+1}]\big | \mathcal{F}_t}\, .
$
First, we have
\begin{eqnarray*}
(W^{t+1} )^{\top} W^{t+1}
= (W^t + \eta^{t+1} s^{t+1} (r^{t+1})^{\top})^{\top} (W^t + \eta^{t+1} s^{t+1} (r^{t+1})^{\top}) \\
= P^t + \eta^{t+1} r^{t+1} (s^{t+1})^{\top} W^t + \eta^{t+1} (W^t)^{\top} s^{t+1} (r^{t+1})^{\top} + (\eta^{t+1})^2 \|s^{t+1}\|^2 r^{t+1} (r^{t+1})^{\top}
\end{eqnarray*}
This implies that
\begin{eqnarray*}
\expec{ tr [U^*(W^{t+1} )^{\top} W^{t+1}]\big | \mathcal{F}_t}
=
tr (U^* \expec{(W^{t+1} )^{\top} W^{t+1}\big | \mathcal{F}_t}) \\
=
tr(U^* P^t) + \eta^{t+1} tr \expec{ U^* r^{t+1} (s^{t+1})^{\top}\big | \mathcal{F}_t} W^t \\
+
\eta^{t+1} tr (\expec{ U^*(W^t)^{\top} s^{t+1} (r^{t+1})^{\top} \big | \mathcal{F}_t}) \\
+
(\eta^{t+1})^2 \expec{ \|s^{t+1}\|^2 tr (U^* r^{t+1} (r^{t+1})^{\top})\big | \mathcal{F}_t } \\
\geq
tr(U^* P^t) + \eta^{t+1} tr U^*\expec{ r^{t+1} (s^{t+1})^{\top}\big | \mathcal{F}_t} W^t \\
+
\eta^{t+1} tr (U^* \expec{ (W^t)^{\top} s^{t+1} (r^{t+1})^{\top} \big | \mathcal{F}_t}) \\
\ge
tr(U^* P^t) + 2\eta^{t+1} tr (U^* \expec{ (W^t)^{\top} s^{t+1} (r^{t+1})^{\top} \big | \mathcal{F}_t})
\end{eqnarray*}
the second to last inequality follows since we can drop the non-negative term, and the last inequality holds since the $tr(A)=tr(A^{\top})$ for any square matrix $A$.
Since
\[
\expec{s^{t+1}(r^{t+1})^{\top} \big | \mathcal{F}_t} = W^t (\Sigma^*- \Sigma^* P^t ) \, ,
\] 
we have
\[
tr U^*\expec{(W^t)^{\top} s^{t+1} (r^{t+1})^{\top}\big | \mathcal{F}_t}
=tr U^*(P^t\Sigma^* -  P^t\Sigma^*P^t) \, .
\]
By Lemma \ref{lm:stationary_pts},
\begin{eqnarray*}
tr U^*\expec{(W^t)^{\top} s^{t+1} (r^{t+1})^{\top} \big | \mathcal{F}_t} \\
=
tr U^*(P^t\Sigma^* -  P^t\Sigma^*P^t) \\
\ge
\sum_{i=1}^k \lambda_i (1-u_i^{\top}P^tu_i)
(u_i^{\top}P^tu_i - \sum_{j\ne i, j\in [k]} [1-u_j^{\top}P^tu_j]) 
\end{eqnarray*}
Then we have, 
\begin{eqnarray*}
\expec{ tr [U^*(W^{t+1} )^{\top} W^{t+1}] \big | \mathcal{F}_t } 
\geq
tr(U^*P^t) \\
+
2\eta^{t+1}\sum_{i=1}^k \lambda_i (1-u_i^{\top}P^tu_i)
(u_i^{\top}P^tu_i - \sum_{j\ne i, j\in [k]} [1-u_j^{\top}P^tu_j])
\end{eqnarray*}
Now we can bound $\expec{tr U^* P^{t+1} \big | \mathcal{F}_t }$ as:
\begin{eqnarray}
\expec{ tr(U^*(W^{t+1})^{\top} [W^{t+1}(W^{t+1})^{\top}]^{-1}W^{t+1}) \big | \mathcal{F}_t } \nonumber\\
\geq
\expec{ tr (U^*(W^{t+1} )^{\top} W^{t+1}) \big | \mathcal{F}_t } 
- 
\expec{ (\eta^{t+1})^2 \|r^{t+1}\|^2 \|s^{t+1}\|^2 tr [U^*(W^{t+1} )^{\top} W^{t+1}] \big | \mathcal{F}_t } 
\nonumber\\
\geq
tr(U^* P^t)
+
2\eta^{t+1}\sum_{i=1}^k \lambda_i (1-u_i^{\top}P^tu_i)
(u_i^{\top}P^tu_i - \sum_{j\ne i, j\in [k]} [1-u_j^{\top}P^tu_j]) \nonumber \\
-
\expec{ (\eta^{t+1})^2 \|r^{t+1}\|^2 \|s^{t+1}\|^2 tr (U^*(W^{t+1} )^{\top} W^{t+1}) \big | \mathcal{F}_t }
~~~\label{ineq:main_interm}
\end{eqnarray}
Note that the second term in the inequality above can be lower bounded as:
\begin{eqnarray*}
\sum_{i=1}^k \lambda_i (1-u_i^{\top}P^tu_i)
(u_i^{\top}P^tu_i - \sum_{j\ne i, j\in [k]} [1-u_j^{\top}P^tu_j]) \\
=
\sum_{i=1}^k \lambda_i (1-u_i^{\top}P^tu_i)
(\sum_{j\in [k]} u_j^{\top}P^tu_j - (k-1)) \\
=
\sum_{i=1}^k \lambda_i (1-u_i^{\top}P^tu_i)
(1-\Delta^t)
\ge
\lambda_k \Delta^t(1-\Delta^t) 
\end{eqnarray*}
Since $k^{\prime}\le d$, and rows of $W^t$ are orthonormal, we get
\[
\|s^{t+1}\|^2
=
\|W^t X^{t+1}\|^2
\le \|X^{t+1}\|^2 \, .
\] 
Similarly, $\|r^{t+1}\|^2 \le  \|X^{t+1}\|^2 \, .$
Therefore,
\begin{eqnarray*}
\|s^{t+1}\|^2 tr (U^*(W^{t+1} )^{\top} W^{t+1}) \\
\le
\|X^{t+1}\|^2
 \bigg(tr U^*P^t + 2\eta^{t+1} tr U^* r^{t+1} (X^{t+1})^{\top} P^t 
 + (\eta^{t+1})^2 \|s^{t+1}\|^2 tr U^* r^{t+1} (r^{t+1})^{\top}\bigg) \\
=
\|X^{t+1}\|^2
 \bigg(tr U^*P^t + 2\eta^{t+1} (X^{t+1})^{\top} P^t  U^* r^{t+1}
 + (\eta^{t+1})^2 \|s^{t+1}\|^2 (r^{t+1})^{\top} U^* r^{t+1} \bigg) \\
\le
\|X^{t+1}\|^2 \bigg( trU^*P^t + 2\eta^{t+1} \|X^{t+1}\|^2 + (\eta^{t+1})^2 \|s^{t+1}\|^2 \|r^{t+1}\|^2\bigg) \\
\le
\|X^{t+1}\|^2 \bigg( trU^*P^t + 2\eta^{t+1} \|X^{t+1}\|^2 + (\eta^{t+1})^2 \|X^{t+1}\|^4 \bigg)
\end{eqnarray*}
On the other hand, we have
\[
\expec{ \|r^{t+1} \|^2\big | \mathcal{F}_t}
= tr(\Sigma^* - \Sigma^*P^t)
\]
Thus, the quadratic term (quadratic in $\eta^{t+1}$) in Eq~\eqref{ineq:main_interm} can be upper bounded as
\begin{eqnarray*}
\expec{ (\eta^{t+1})^2 \|r^{t+1}\|^2 \|s^{t+1}\|^2 tr (U^*(W^{t+1} )^{\top} W^{t+1}) \big | \mathcal{F}_t }\\
\leq
(\eta^{t+1})^2 C^t_o \expec{ \|r^{t+1} \|^2\big | \mathcal{F}_t}
=
(\eta^{t+1})^2 C^t_o tr(\Sigma^* - \Sigma^*P^t)
\end{eqnarray*}
where 
\begin{eqnarray*}
C^t_o\Defeq 
\max_X \|X\|^2 \big(trU^*P^t + 2\eta^{t+1} \|X\|^2 + (\eta^{t+1} \|X\|^2)^2\big) \\
\le
\max_X \|X\|^2 \big(k + 2\eta^{t+1} \|X\|^2 + (\eta^{t+1} \|X\|^2)^2\big)\\
=
kb + 2\eta^{t+1}b^2 + (\eta^{t+1})^2b^3
\end{eqnarray*}
Note that by our definition of $C^{t+1}$, 
\[
C^{t+1}\Defeq kb + 2\eta^{t+1}b^2 + (\eta^{t+1})^2b^3
\]
We get that
\begin{eqnarray*}
\expec{ (\eta^{t+1})^2 \|r^{t+1}\|^2 \|s^{t+1}\|^2 tr (U^*(W^{t+1} )^{\top} W^{t+1}) \big | \mathcal{F}_t }
\le
C^{t+1} (\eta^{t+1})^2 tr(\Sigma^* - \Sigma^*P^t) \, .
\end{eqnarray*}
Since 
\[
\lambda_1 U^* \succeq \Sigma^*  \succeq \lambda_k U^*
\]
We have
\begin{eqnarray*}
(I-P^t)^{\top}\lambda_1 U^*  (I-P^t)
\succeq 
(I-P^t)^{\top} \Sigma^*  (I-P^t) \\
\succeq 
(I-P^t)^{\top} \lambda_k U^* (I-P^t)
\end{eqnarray*}
Note that the projection matrix satisfies 
$(I-P^t)^{\top}=(I-P^t)$ 
and 
$(I-P^t)(I-P^t) = (I-P^t)$.
This implies that
\begin{eqnarray}
\label{eqn:loss_measure}
\lambda_1 tr U^*  (I-P^t)
\ge
tr \Sigma^*  (I-P^t)
\ge
\lambda_k  tr U^*  (I-P^t)
\end{eqnarray}
Finally, plug the lower bound in Eq. \eqref{ineq:main_interm} completes the proof:
\begin{eqnarray*}
\expec{tr(U^* P^{t+1})|\mathcal{F}_t} 
\geq 
tr(U^* P^t)  
+
\lambda_k \Delta^t(1-\Delta^t)  
-
(\eta^{t+1})^2 C^{t+1} \lambda_1 tr(U^*(I-P^t))\\
\ge
tr(U^* P^t)  
+
\lambda_k \Delta^t(1-\Delta^t) 
-
(\eta^{t+1})^2 C^{t+1} \lambda_1 \Delta^t\\
\end{eqnarray*}
\newline
The inequality of the statement in Version 2 can be obtained similarly, by setting
\[
Z  = 2 \bigg( tr(U^*(W^t)^{\top}s^{t+1} (r^{t+1})^{\top}) - \expec{ tr(U^*(W^t)^{\top}s^{t+1} (r^{t+1})^{\top}) \big | \mathcal{F}_t} \bigg)
\]
It is clear that $\expec{Z\big | \mathcal{F}_t} = 0$.
Now we upper bound $|Z|$: Since
\[
tr(U^*(W^t)^{\top}s^{t+1} (r^{t+1})^{\top})
= tr U^* P^t X^{t+1}(X^{t+1})^{\top} (I-P^t)
\]
we get (subsequently, we denote $P^t$ by $P$, $X^{t+1}$ by $X$)
\begin{eqnarray*}
|Z|  = |2tr (U^* PXX^{\top} (I-P)) - 2 tr(U^* P \Sigma^* (I-P))| \\
=
2|tr(XX^{\top} -\Sigma^*)(I-P)U^* P|
\leq
2 \sqrt{\|XX^{\top} - \Sigma^*\|_F^2 \|(I-P)U^* P\|_F^2}
\end{eqnarray*}
We first bound $\|(I-P)U^* P\|_F^2$, 
\[
\|(I-P)U^* P\|_F^2
\le \|(I-P)U^*\|_F^2
= tr(U^*-U^* P)
\]
where the first inequality is due to the fact that $P$ is a projection matrix so that norms are at best preserved if not smaller;
the second inequality is also due to the fact that both $U^*$ and $I-P$ are projection matrices, and thus $(I-P)(I-P)=I-P$ and $U^*U^* = U^*$. 
Now we bound $\|XX^{\top} - \Sigma^*\|_F^2$:
\begin{eqnarray*}
\|XX^{\top} - \Sigma^*\|_F^2
= tr (XX^{\top} - \Sigma^*)^{\top} (XX^{\top} - \Sigma^*) \\
= \|X\|^4 - 2X^{\top}\Sigma^*X + \|\Sigma^*\|_F^2 
\leq
\|X\|^4 +  \|\Sigma^*\|_F^2
\end{eqnarray*}
where the last inequality is due to the fact that $\Sigma^*$ is positive semidefinite, that is, for any $x$, we have
$
x^{\top}\Sigma^*x \ge 0
$.
Finally,
\begin{eqnarray*}
|Z|
\leq
2 \sqrt{\|XX^{\top} - \Sigma^*\|_F^2 \|(I-P)U^* P\|_F^2} \\
\leq
2 \sqrt{(\|X\|^4 +  \|\Sigma^*\|_F^2) tr(U^*-U^* P)} \\
\le
2 (\|X\|^2 +  \|\Sigma^*\|_F) \sqrt{\Delta^t}
\le
2 (b +  \|\Sigma^*\|_F) \sqrt{\Delta^t}
\end{eqnarray*}
The third inequality is by the following argument: for any $a\ge 0, b\ge 0$, we have
$
\sqrt{a^2+b^2} \le a + b
$,
Letting $a = \|X\|^2$ and $b =  \|\Sigma^*\|_F$ leads to the inequality.
\end{proof}

\subsection{Auxiliary lemmas for Proposition~\ref{prop:iter_wise}}
%
\begin{proof}[Proof of Lemma \ref{lm:inverse_matrix_bound}]
\begin{eqnarray*}
W^{t+1} (W^{t+1} )^{\top}
= (W^t + \eta^{t+1} s^{t+1} (r^{t+1})^{\top})(W^t + \eta^{t+1} s^{t+1} (r^{t+1})^{\top})^{\top} \\
=
(W^t)(W^t)^{\top}
+ \eta^{t+1} s^{t+1} (r^{t+1})^{\top} (W^t)^{\top} \\
+ \eta^{t+1} W^t r^{t+1} (s^{t+1})^{\top}
+ (\eta^{t+1})^2 \|r^{t+1}\|^2 s^{t+1} (s^{t+1})^{\top} \\
= I_{k^{\prime}} + (\eta^{t+1})^2 \|r^{t+1}\|^2 s^{t+1} (s^{t+1})^{\top}
\end{eqnarray*}
where the last equality holds because $W^t$ has orthonormalized rows, and $r^{t+1}$ is orthogonal to rows of $W^t$.
Let 
\[
E \Defeq (\eta^{t+1})^2 \|r^{t+1}\|^2 s^{t+1} (s^{t+1})^{\top} \, .
\] 
Note that $E$ is symmetric and positive semidefinite.
We can eigen-decompose $E$ as
\[
E = Q \Lambda Q^{\top}
\]
where $Q$ is the eigenbasis and $\Lambda$ is a diagonal matrix with real non-negative diagonal values, with $\Lambda_{11}\ge \Lambda_{22}, \ge , \dots \Lambda_{k^{\prime}k^{\prime}}$, corresponding to the non-decreasing eigenvalues of $E$.
Then
\[
(I_{k^{\prime}} + E)^{-1} 
= 
(QQ^{\top} + Q \Lambda Q^{\top})^{-1}
=
Q (I_{k^{\prime}} + \Lambda)^{-1} Q^{\top} \, ,
\]
Since $I_{k^{\prime}} + \Lambda$ is a diagonal matrix, for any $i\in [k^{\prime}]$, we have
\[
(I_{k^{\prime}} + \Lambda)^{-1}_{ii} = \frac{1}{1+\Lambda_{ii}}
\geq
1-\Lambda_{ii}
\geq
1-\Lambda_{11}
\]
This implies that the matrix
\[
Q
[(I_{k^{\prime}} + \Lambda)^{-1} - (1-\Lambda_{11})I_{k^{\prime}}]
Q^{\top}
\]
is positive semidefinite, that is,
\[
Q (I_{k^{\prime}} + \Lambda)^{-1} Q^{\top}
\succeq
Q (1-\Lambda_{11})I_{k^{\prime}} Q^{\top}
=
(1-\Lambda_{11})I_{k^{\prime}}
\]
Thus,
\begin{eqnarray*}
(W^{t+1} (W^{t+1} )^{\top})^{-1}
=
(I_{k^{\prime}} + E)^{-1}
\succeq
(1-\Lambda_{11})I_{k^{\prime}}
\end{eqnarray*}
Finally, we compute the largest eigenvalue of $E$, $\lambda_1(E)\Defeq\Lambda_{11}$:
\begin{eqnarray*}
\lambda_1(E)
=
\max_{\|y\|=1} y^{\top} E y
=
\max_{\|y\|=1} (\eta^{t+1})^2 \|r^{t+1}\|^2 (y^{\top} s^{t+1} (s^{t+1})^{\top} y) \\
=
(\eta^{t+1})^2 \|r^{t+1}\|^2 \max_{\|y\|=1} (\langle s^{t+1}, y\rangle)^2
=
(\eta^{t+1})^2 \|r^{t+1}\|^2 \|s^{t+1}\|^2 
\end{eqnarray*}
This completes the proof.
\end{proof}

\begin{proof}[Proof of Lemma~\ref{lm:stationary_pts}]
We first prove \underline{statement 1}.
Since $U^*$ is symmetric and positive semidefinite, we can write it as $U^* = ((U^*)^{1/2})^2$.
So we have
\begin{eqnarray*}
tr(U^* - U^*P^t)
= tr (U^*(I-P^t)) \\
=tr ((U^*)^{1/2}(I-P^t)(I-P^t)(U^*)^{1/2}) 
= \|(I-P^t)(U^*)^{1/2}\|_F^2
\end{eqnarray*}
Therefore,
$
tr(U^*) = tr (U^*P^t)
$
implies that
\[
tr(U^* - U^*P^t) = \|(I-P^t)(U^*)^{1/2}\|_F^2 = 0
\]
which implies
\[
(I-P^t)(U^*)^{1/2} = 0
\]
where ``$0$'' denotes the zero matrix.
Thus,
\[
\Gamma^t
=
tr (P^t \Sigma^* (I-P^t)U^*)
= tr (P^t \Sigma^*(I-P^t)(U^*)^{1/2} (U^*)^{1/2})
= tr 0 = 0 \, .
\]
Now we prove \underline{statement 2}.
First, we upper bound $tr(P^t\Sigma^*P^tU^*)$:
\begin{eqnarray*}
tr(P^t\Sigma^*P^tU^*)
=tr(\sum_{p=1}^{k^{\prime}}\sum_{i=1}^k \sum_{q=1}^{k^{\prime}}\sum_{j=1}^k
\lambda_i\langle w_{p}, u_i\rangle
\langle w_q, u_i\rangle
\langle w_q, u_j\rangle
w_{p} u_j^{\top}
) 
\\
=
\sum_i \lambda_i
\sum_j
\sum_{p}
\langle w_{p}, u_i\rangle \langle w_{p}, u_j\rangle
\sum_q
\langle w_{q}, u_i\rangle \langle w_{q}, u_j\rangle
\\
= 
\sum_i \lambda_i\sum_j (u_i^{\top} P^t u_j)^2
\end{eqnarray*}
Note that by Cauchy-Schwarz inequality,
\[
(u_i^{\top} P^t u_j)^2
=
(u_i^{\top} P^t (P^t)^{\top} u_j)^2
\le
\|P^t u_i\|^2\|P^t u_j\|^2
=
(u_i^{\top} P^t u_i)(u_j^{\top} P^t u_j)
\]
On the other hand, for any $i$ and $j\ne i$ since $u_i \perp u_j$, we have
$$
u_i^{\top}P^t u_j
=
u_i^{\top}u_j - u_i^{\top} (I-P^t) u_j
=
-u_i^{\top} (I-P^t) u_j
$$
we have
\begin{eqnarray*}
(u_i^{\top}P^t u_j)^2 
= (u_i^{\top} (I-P^t) u_j)^2 
= (u_i^{\top} (I-P^t)(I-P^t) u_j)^2\\
\le \|(I-P^t) u_i\|^2 \|(I-P^t) u_j\|^2 \\
=
(\|u_i\|^2 - \|P^tu_i\|^2)(\|u_j\|^2 - \|P^tu_j\|^2)\\
=
(1-u_i^{\top}P^tu_i)(1-u_j^{\top}P^tu_j)
\end{eqnarray*}
where the inequality is by Cauchy-Schwarz inequality,
and the third equality is by combining orthogonality of projection $P^t$ and Pythagorean theorem.
This implies that
\begin{eqnarray*}
tr(P^t \Sigma^* P^tU^*)
=
\sum_i \lambda_i \sum_j (u_i^{\top} P^t u_j)^2 \\
=
\sum_i \lambda_i(u_i^{\top} P^t u_i)^2
+
\sum_i \lambda_i\sum_{j\ne i} (u_i^{\top} P^t u_j)^2 \\
\leq
\sum_i \lambda_i(u_i^{\top} P^t u_i)^2
+
\sum_i \lambda_i \sum_{j\ne i} (1-u_i^{\top}P^tu_i)(1-u_j^{\top}P^tu_j)
\end{eqnarray*}
Next, we expand $tr(P^t\Sigma^*U^*)$:
\begin{eqnarray*}
tr(P^t\Sigma^*U^*) = tr(U^*P^t\Sigma^*)
=
tr(\sum_i u_i u_i^{\top} P^t \sum_j\lambda_j u_j u_j^{\top}) \\
=
\sum_i\sum_j \lambda_j u_i^{\top} P^t u_j u_i^{\top} u_j 
= \sum_i \lambda_i u_i^{\top} P^t u_i
\end{eqnarray*}
Combining the upper bound on $tr(P^t \Sigma^* P^t U^*)$, we get,
\begin{eqnarray*}
tr(P^t\Sigma^*U^*) - tr(P^t \Sigma^*P^tU^*) 
=
\sum_i \lambda_i u_i^{\top} P^t u_i - tr(P^t \Sigma^* P^tU^*) \\
\ge
\sum_i \lambda_i u_i^{\top} P^t u_i
-
\sum_i \lambda_i (u_i^{\top} P^t u_i)^2
-
\sum_i\lambda_i \sum_{j\ne i} (1-u_i^{\top}P^tu_i)(1-u_j^{\top}P^tu_j)
\\
=
\sum_i \lambda_i (1-u_i^{\top}P^tu_i)
(u_i^{\top}P^tu_i - \sum_{j\ne i} [1-u_j^{\top}P^tu_j]) 
\end{eqnarray*}
Recall that
\[
\Delta^t = k - tr(U^*P^t)
= k - \sum_{i=1}^k u_i^{\top} P^t u_i \, ,
\]
Therefore, the last term in the inequality above can be further lower bounded by
$
\lambda_k \Delta^t (1-\Delta^t) \, .
$
\end{proof}
\section{Proof of Theorem \ref{thm:main}}
\label{sec:proof_thm}
\begin{proof}[Proof of Theorem \ref{thm:main}]
Since by our assumption, 
$
\Delta^o \le \frac{1-\tau}{2},
$
for any $\delta>0$, and
since we choose the learning rate such that
\[
\eta
\le
\min\{\frac{2\lambda_k\tau}{\frac{16}{1-\tau}\ln\frac{1}{\delta}(b+\|\Sigma^*\|_F)^2 + b(k+1)\lambda_1},\frac{\sqrt{2}-1}{b} \} \, ,
\]
we can
apply Proposition \ref{prop:bad_event} to bound the probability of bad event, $\mathcal{G}_t^c$ as
$
\proba{\mathcal{G}_t^c} \le
\delta \, .
$
By V1 of Proposition \ref{prop:iter_wise} (and let $C^{t+1}$ be as denoted therein), 
\begin{eqnarray*}
\expec{tr(U^* P^{t+1})|\mathcal{F}_t} 
\geq 
tr(U^* P^t)  
+
2\eta^{t+1} \lambda_k \Delta^t (1-\Delta^t)
-
(\eta^{t+1})^2 C^{t+1} \lambda_1\Delta^t \, ,
\end{eqnarray*}
Rearranging the inequality above and adding $k$ to both sides,
\begin{eqnarray*}
\expec{\Delta^{t+1} | \mathcal{F}_t} 
\le
\Delta^t 
- 2\eta^{t+1}\lambda_k \Delta^t (1-\Delta^t)
+
(\eta^{t+1})^2 C^{t+1} \lambda_1 \Delta^t \\
=
\Delta^t \bigg(1 - 2\eta^{t+1}\lambda_k(1-\Delta^t) + (\eta^{t+1})^2 C^{t+1} \lambda_1\bigg) \, ,
\end{eqnarray*}
Multiplying both sides of the inequality above by $\indic{\mathcal{G}_t}$, we get
\[
\expec{\Delta^{t+1} | \mathcal{F}_t} \indic{\mathcal{G}_t}
\le
\Delta^t \bigg(1 - 2\eta^{t+1}\lambda_k(1-\Delta^t) + (\eta^{t+1})^2 C^{t+1} \lambda_1\bigg)\indic{\mathcal{G}_t} \, ,
\]
Since $\mathcal{G}_t$ is $\mathcal{F}_t$-measurable, we have
\[
\expec{\Delta^{t+1}\indic{\mathcal{G}_t} | \mathcal{F}_t} 
=
\expec{\Delta^{t+1} | \mathcal{F}_t} \indic{\mathcal{G}_t} \, ,
\]
When $\indic{\mathcal{G}_t}=1$, we have
$
1-\Delta^t
\ge
\tau \, .
$
Therefore,
\begin{eqnarray*}
\expec{\Delta^{t+1}\indic{\mathcal{G}_t} | \mathcal{F}_t} 
\le
\Delta^t \bigg(1 - 2\eta^{t+1}\lambda_k\tau + (\eta^{t+1})^2 C^{t+1} \lambda_1\bigg)\indic{\mathcal{G}_t} \\
\le
\Delta^t \bigg(1 - 2\eta^{t+1}\lambda_k\tau + (\eta^{t+1})^2 C^{t+1} \lambda_1\bigg)\indic{\mathcal{G}_{t-1}}
\end{eqnarray*}
where the last inequality holds since $\mathcal{G}_t\subset \mathcal{G}_{t-1}$.
Taking expectation over both sides, we get the following recursion relation:
\[
\expec{\Delta^{t+1}\indic{\mathcal{G}_t}} 
\le
\expec{\Delta^t \indic{\mathcal{G}_{t-1}}} 
 \bigg(1 - 2\eta^{t+1}\lambda_k\tau + (\eta^{t+1})^2 C^{t+1} \lambda_1\bigg)
\]
We further bound
$
1- 2\eta^{t+1} \tau\lambda_k+ (\eta^{t+1})^2 C^{t+1} \lambda_1 \, .
$
First, note that since we require
$
\eta^{t+1} \le \frac{\lambda_k\tau}{\lambda_1 b(k+3)} \, ,
$
we get
\[
\eta^{t+1} b
\le \frac{\lambda_k\tau}{\lambda_1 (k+3)}
\le
\frac{\tau}{(k+3)}
\le
\frac{1}{k+3}
\le
\frac{1}{4} \, .
\]
and
\[
C^{t+1}
=
b(k+2\eta^{t+1}b + (\eta^{t+1})^2b^2)
\le
b (k+1) \, .
\]
Thus, we get
\begin{eqnarray*}
1- 2\eta^{t+1} \tau\lambda_k+ (\eta^{t+1})^2 C^{t+1} \lambda_1 
\le
1- 2\eta^{t+1} \tau\lambda_k+ (\eta^{t+1})^2 b(k+1) \lambda_1
\end{eqnarray*}
Since our requirement of $\eta^{t+1}$ also implies that
\[
\eta^{t+1} \le
\frac{2\lambda_k\tau}{b(k+1)\lambda_1} \, ,
\] 
it
guarantees that
\[
0 < 1- 2\eta^{t+1} \tau\lambda_k+ (\eta^{t+1})^2 b(k+1) \lambda_1 \, .
<
1
\]
For any $t$, define 
$
\alpha^t\Defeq 
2\eta^{t} \tau\lambda_k- (\eta^{t})^2 b(k+1) \lambda_1 \, ,
$
we have
\[
\expec{\Delta^{t+1}\indic{\mathcal{G}_t}}
\le
\expec{\Delta^{t}\indic{\mathcal{G}_{t-1}}}\big(1- \alpha^{t+1}\big) \, ,
\]
Recursively applying this relation, we get
\[
\expec{\Delta^{t+1}\indic{\mathcal{G}_t}}
\le
\Pi_{i=2}^{t+1}(1-\alpha^i) \expec{\Delta^1 \indic{\mathcal{G}_0}}
\]
Also note that
\[
\Delta^1 \indic{\mathcal{G}_0}
\le
(1-\alpha^1) \Delta^0 \, ,
\]
Therefore,
\[
\expec{\Delta^{t+1}\indic{\mathcal{G}_t}}
\le
\Pi_{i=1}^{t+1}(1-\alpha^i) \Delta^0
\]
Since for any $x\in (0, 1)$, it holds that
$\ln (1-x) \le -x$; we get
\[
 \Pi_{i=1}^t \big(1-\alpha^{i}\big)
 \le
 \exp{-\sum_{i=1}^t \alpha^{i}}
\]
Plugging in the value of $\alpha^i$'s, we get
\[
\expec{\Delta^{t}\indic{\mathcal{G}_{t-1}}}
\le
 \exp{-\sum_{i=1}^t \bigg(2\eta^{i} \tau\lambda_k- (\eta^{i})^2 b(k+1) \lambda_1\bigg)}
\]
Again, by our requirement on learning rate, we have for any $t$
\[
\eta^t \le \frac{\lambda_k\tau}{\lambda_1 b (k+3)} \le \frac{\lambda_k\tau}{\lambda_1 b (k+1)}
\]
Thus,
\[
2\eta^{i} \tau\lambda_k- (\eta^{i})^2 b(k+1) \lambda_1
\ge
\eta^i \tau \lambda_k > 0
\]
Since we choose a constant learning rate $\eta$, this implies that
\[
\expec{\Delta^{t}\indic{\mathcal{G}_{t-1}}}
\le
\exp{-\sum_{i=1}^t \eta \tau \lambda_k}
=
\exp{- t \eta \tau \lambda_k}
\]
Finally, since $\indic{\mathcal{G}_{t}} \le \indic{\mathcal{G}_{t-1}}$, we get
\[
\expec{\Delta^{t}\indic{\mathcal{G}_{t}}}
\le
\expec{\Delta^{t}\indic{\mathcal{G}_{t-1}}}
\le
\exp{- t \eta \tau \lambda_k}
\]
Combining this with the definition of conditional expectation, we get
\[
\expec{\Delta^t | \mathcal{G}_{t}}
\Defeq
\frac{\expec{\Delta^{t}\indic{\mathcal{G}_{t}}}}{\proba{\mathcal{G}_{t}}}
\le
\frac{\expec{\Delta^{t}\indic{\mathcal{G}_{t}}}}{1-\delta}
\le
\frac{1}{1-\delta} \exp{- t \eta \tau \lambda_k}
\]
where the first inequality is by our upper bound on the probability of bad event $\mathcal{G}_{t}^c$.
\end{proof}
\section{Canonical (principal) angles between subspaces}
\begin{defn}[\citet{Vu:minimax-subspace:2013}] 
Let $\mathcal{E}$ and $\mathcal{F}$ be $d$-dimensional subspaces of $\mathbb{R}^p$ with orthogonal projectors
$E$ and $F$. Denote the singular values of $EF^{\perp}$ by $s_1\ge s_2\dots \ge$. The canonical angles between $\mathcal{E}$ and $\mathcal{F}$ are the numbers
\[
\theta_k(\mathcal{E}, \mathcal{F}) = \arcsin (s_k)
\] 
for $k=1,\dots, d$ and the angle operator between $\mathcal{E}$ and $\mathcal{F}$ is the $d\times d$ matrix
\[
\Theta (\mathcal{E}, \mathcal{F}) = diag(\theta_1,\dots, \theta_d) \, .
\]
subject to
\[
\|x\| = \|y\| = 1, x^H x_i = 0, y^H y_i = 0, i = 1, \dots , k - 1.
\]
The vectors $\{x_1, \dots , x_m\}$ and $\{y_1, \dots , y_m\}$ are called the principal vectors.
\end{defn}
\begin{prop}
Let $\mathcal{E}$ and $\mathcal{F}$ be $d$-dimensional subspaces of $\mathbb{R}^p$ with orthogonal projectors
$E$ and $F$. Then
The singular values of $EF^{\perp}$ are
\[
s_1, s_2, \dots, s_d, 0, \dots, 0.
\]
And
\[
\|\sin \Theta (\mathcal{E}, \mathcal{F})\|_F^2
=
\|E F^{\perp}\|_F^2 \, .
\]
\end{prop}

\end{document}